\documentclass{article}

\usepackage{microtype}
\usepackage{graphicx}
\usepackage{subcaption}
\usepackage{booktabs} 
\usepackage{makecell}
\usepackage{hyperref}

\usepackage{pifont}
\newcommand{\cmark}{\ding{51}}
\newcommand{\xmark}{\ding{55}}
\newcommand{\eqD}{\mathrel{\stackrel{\mathrm{D}}{=}}}

\usepackage[linesnumbered,ruled, vlined]{algorithm2e}

\usepackage[accepted]{icml2026}

\usepackage{amsmath}
\usepackage{amssymb}
\usepackage{mathtools}
\usepackage{amsthm}
\usepackage{bbm}
\usepackage{thmtools,thm-restate}

\usepackage[capitalize,noabbrev]{cleveref}
\hypersetup{hypertexnames=false}

\theoremstyle{plain}
\newtheorem{theorem}{Theorem}[section]
\newtheorem{proposition}[theorem]{Proposition}
\newtheorem{lemma}[theorem]{Lemma}
\newtheorem{corollary}[theorem]{Corollary}
\theoremstyle{definition}
\newtheorem{definition}[theorem]{Definition}
\newtheorem{assumption}[theorem]{Assumption}
\theoremstyle{remark}

\newcommand{\E}{\mathbb{E}}

\icmltitlerunning{Distributional Inverse Reinforcement Learning}

\begin{document}

\twocolumn[
  \icmltitle{Distributional Inverse Reinforcement Learning}


  \icmlsetsymbol{equal}{*}

  \begin{icmlauthorlist}
    \icmlauthor{Feiyang Wu}{gtcse}
    \icmlauthor{Ye Zhao}{gtme}
    \icmlauthor{Anqi Wu}{gtcse}
  \end{icmlauthorlist}

  \icmlaffiliation{gtcse}{School of Computational Science and Engineering, Georgia Institute of Technology, Atlanta, USA}
  \icmlaffiliation{gtme}{George W. Woodruff School of Mechanical Engineering, Georgia Institute of Technology, Atlanta, USA}

  \icmlcorrespondingauthor{Feiyang Wu}{feiyangwu@gatech.edu}

  \icmlkeywords{Reinforcement Learning, Neuroscience, Robotics}

  \vskip 0.3in
]



\printAffiliationsAndNotice{}  

\begin{abstract}
  We propose a distributional framework for offline Inverse Reinforcement Learning (IRL) that jointly models uncertainty over reward functions and full distributions of returns. Unlike conventional IRL approaches that recover a deterministic reward estimate or match only expected returns, our method captures richer structure in expert behavior, particularly in learning the reward distribution, by minimizing first-order stochastic dominance (FSD) violations and thus integrating distortion risk measures (DRMs) into policy learning, enabling the recovery of both reward distributions and distribution-aware policies. This formulation is well-suited for behavior analysis and risk-aware imitation learning. Theoretical analysis shows that the algorithm converges with $\mathcal{O}(\varepsilon^{-2})$ iteration complexity. Empirical results on synthetic benchmarks, real-world neurobehavioral data, and MuJoCo control tasks demonstrate that our method recovers expressive reward representations and achieves state-of-the-art performance.
\end{abstract}

\section{Introduction}

Inverse Reinforcement Learning (IRL) aims to infer an expert's underlying reward function and policy from observed trajectories collected under unknown dynamics. IRL has been successfully applied in diverse domains, including robotics \citep{vasquez2014inverse, wu2024infer}, animal behavior modeling \citep{ashwood2022dynamic, ke2025inverse}, autonomous driving \citep{rosbach2019driving, wu2020efficient}, and fine-tuning of large language models \citep{zeng2025demonstrations}. A pioneering work in this field, the Maximum Entropy IRL (MaxEntIRL) framework \citep{ziebart2008maximum}, formulates reward learning as a likelihood optimization problem and interprets expert policies as Boltzmann distributions over returns. Follow-up works have extended this framework to improve reward inference stability and generalization \citep{arora2021survey, garg2021iq, zeng2022maximum}.

Despite these advances, most IRL methods assume that the expert's reward function is deterministic, thereby recovering only a point estimate, i.e., $r(s,a)\in \mathbb{R}$ for every state $s$ and action $a$. This assumption, however, limits expressiveness in real-world settings where reward signals are inherently stochastic. For instance, in robotic manipulation tasks involving deformable or fragile objects \citep{yin2021modeling}, contact uncertainty introduces reward variability for identical state-action pairs--variability that directly influences the learned policy's robustness and safety. Similarly, in neuroscience, dopaminergic activity has been shown to act as a reward-related teaching signal that shapes animal behavior via RL-like mechanisms \citep{markowitz2023spontaneous}. Yet, dopamine signals exhibit significant trial-to-trial variations, suggesting that behavior may arise from an underlying stochastic reward distribution. These challenges are further amplified in offline IRL settings, where interaction with the environment is unavailable, and the algorithm must fully rely on fixed demonstrations.

These examples highlight that in many real-world scenarios, demonstrations may be generated under stochastic reward functions, i.e., $r(s,a)$ is a random variable. This motivates the need to go beyond point estimates and instead recover the full distribution of rewards. Prior works such as Bayesian IRL (BIRL) methods infer a posterior over reward parameters using Markov chain Monte Carlo (MCMC) \citep{ramachandran2007bayesian}, maximum a posteriori (MAP) estimation \citep{choi2011map}, or variational inference \citep{chan2021scalable}, but primarily capture uncertainty over the parameters of a deterministic reward function. Their policy likelihoods or optimality models are still driven by expected-return quantities, such as soft optimal $Q$-values, rather than by the full reward-induced return distribution. Consequently, higher-order reward structure can be invisible to the objective: two reward distributions with the same mean can induce the same expectation-level signal while differing substantially in variance, skewness, or tail behavior.

However, it remains unclear how to effectively learn reward distributions directly from expert demonstrations. A natural alternative is to compare return distributions using statistical distances such as Wasserstein distance. Such distances are useful for matching two distributions, but they do not by themselves encode the IRL preference relation that the expert should be better than the learner under the recovered reward. In other words, distribution matching and distributional ordering are distinct requirements. Consequently, a principled framework is needed to lift the expert-preference condition from expected returns to full return distributions while simultaneously supporting return-distribution estimation in the offline IRL setting.

To this end, we introduce \textit{Distributional Inverse Reinforcement Learning} (DistIRL), a novel framework that explicitly models both the distributional nature of reward and the return. 
This allows us to capture stochasticity not only from transitions and policies but also from the reward function itself. 
Specifically, for reward learning, instead of matching expected returns as in MaxEntIRL, we propose to compare full return distributions using a First-order Stochastic Dominance (FSD) criterion. This allows us to capture not only the mean but also higher-order moments of the return distribution and thus capture the full landscape of reward distributions, leading to a richer and more faithful estimate of the underlying reward structure. To the best of our knowledge, \textit{this is the first offline IRL framework to learn full reward distributions in a principled manner while also learning distribution-aware policies.} 

It is important to note that while our framework incorporates risk-sensitive policy learning, risk sensitivity primarily serves as a mechanism that enables robust reward distribution learning in the offline IRL setting. The connection is explained in detail in Sec.~\ref{sec:risk-aware-policy-learning}.
Our contributions in this paper are summarized as follows:\\
(1) \textbf{Reward Distribution Learning.} We propose an intuitive framework for learning reward distributions in the offline IRL setting. With an FSD objective emphasizing the entire distribution, we are able to learn reward distributions beyond the first moment.\\
(2) \textbf{Distribution-aware Policy Learning.} Our algorithm learns the return distribution and recovers the distribution-aware policy, extending the modeling capability of IRL frameworks towards a broader range of behavior analyses and facilitating imitation learning in risk-sensitive scenarios.\\
(3) \textbf{Theoretical Analysis.} We develop a convergence-rate analysis for the proposed algorithm for solving DistIRL, showing that the algorithm converges with $\mathcal{O}(\varepsilon^{-2})$ iteration complexity.\\
(4) \textbf{Empirical Validation.} We demonstrate that our method recovers meaningful reward distributions on synthetic and real-world datasets, including neurobehavioral data. Our algorithm also achieves state-of-the-art performance on high-dimensional robotic control tasks in offline IRL settings. 
\section{Related Work}

\paragraph{Inverse Reinforcement Learning}
Traditional offline IRL algorithms recover a reward function by matching expert feature expectations or maximizing an entropy-regularized likelihood. Apprenticeship learning \citep{abbeel2004apprenticeship} and maximum entropy / maximum causal entropy IRL \citep{ziebart2008maximum, ziebart2010modeling, gleave2022primer} infer a deterministic reward whose induced policy reproduces expert behavior in expectation. Subsequent deep IRL and imitation-learning variants incorporate neural function approximators or adversarial objectives \citep{ho2016generative, jeon2018bayesian, wulfmeier2015maximum, ni2021f, garg2021iq, zeng2022maximum, viano2021robust, bloem2014infinite, wu2024diffusing, zhan2024model}. Several of these methods require online interaction with a simulator or a learned dynamics model during training, which is undesirable or infeasible in strictly offline settings such as modeling animal behavior from fixed recordings. Recent offline or model-based IRL approaches \citep{zeng2023demonstrations, kostrikov2019imitation} reduce this dependence, but still operate primarily through deterministic rewards or expectation-level matching.
Finally, while recent work has explored risk-aware policy learning within the IRL framework \citep{singh2018risk, lacotte2019risk, cheng2023eliciting,bashiri2021distributionally}, these approaches do not infer a stochastic reward distribution itself. We show a detailed comparison across modeling assumptions in Appendix~\ref{tab:lit_review}.

\paragraph{Bayesian Inverse Reinforcement Learning}
Bayesian IRL (BIRL) methods infer a posterior distribution over reward parameters to quantify uncertainty in reward estimation. Ramachandran and Amir \citep{ramachandran2007bayesian} introduce Bayesian IRL using MCMC to sample from a reward posterior under a Boltzmann-rationality likelihood. Follow-up works use related frameworks to handle larger state spaces and richer reward priors \citep{choi2011map, levine2011nonlinear, chan2021scalable, li2023internally}. Although these methods capture parameter uncertainty, they still rely on expected-return optimality models and do not exploit the full reward-induced return distribution. In continuous action spaces, exact Boltzmann likelihoods can also be difficult to normalize or differentiate. In contrast, we propose a scalable algorithm for learning full reward distributions through a variational distributional objective.

\paragraph{Distributional Reinforcement Learning} 
DistRL extends classical value-based methods by modeling the full distribution of returns rather than only their expectation. Early work, such as Categorical DQN (C51) \citep{bellemare2017distributional} and Quantile Regression DQN (QR-DQN) \citep{dabney2018distributional}, demonstrates that learning a distributional critic improves stability and sample efficiency. More recent advances include Implicit Quantile Networks (IQN) \citep{dabney2018implicit}, Implicit Q-Learning \citep{kostrikov2021offline}, Multivariate Distribution RL \citep{wiltzer2024foundations}, and diffusion processes for RL \citep{hansen2023idql, li2024bellman}. Standard DistRL still typically optimizes expected return after learning a distributional critic. Risk-sensitive extensions \citep{lim2022distributional, schneider2024learning} optimize risk measures such as CVaR, showing that policies can be shaped by emphasizing specific regions of the return distribution. Related work also studies policy families for multiple or variable risk measures, including distributional Pareto-optimal multi-objective RL \citep{cai2023distributionalpareto} and risk-conditioned RL \citep{yoo2024riskconditioned}. These directions are complementary to ours: they adapt forward RL policies to risk preferences, whereas DistIRL infers a reward distribution from demonstrations.

Recent imitation-learning work also studies matching the expert's return distribution directly \citep{lazzati2026returnmatching}. This is an important adjacent direction, but it assumes a known reward in its main formulation and targets return-distribution matching for policy learning. Our problem is different: the reward is unknown, and the central object to infer is the reward distribution itself. IRL counterparts using distributional critics \citep{lee2022risk, karimirize} remain limited in scope because they still assume deterministic reward functions and follow a MaxEntIRL-style mean-matching blueprint.
\begin{figure}[h!]
\centering
\includegraphics[width=0.8\linewidth]{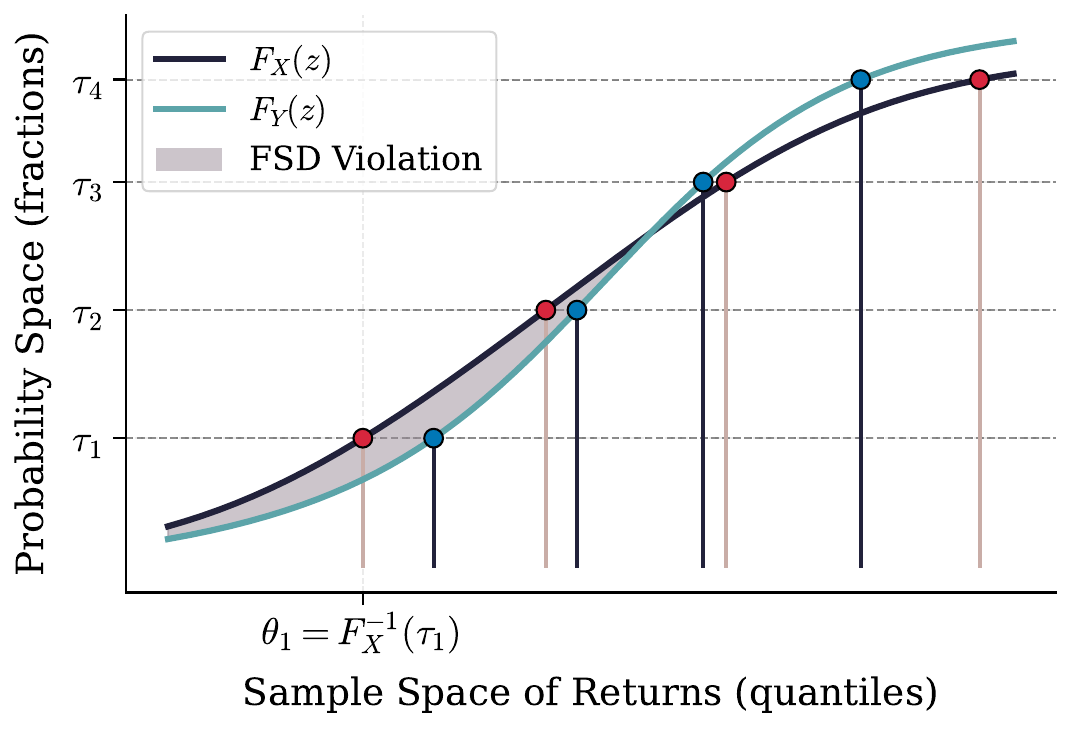}
\caption{Illustration of quantile functions and first-order stochastic dominance (FSD).}
\label{fig:quantile-example}
\end{figure}

\section{Preliminaries}
We model the environment as a discounted Markov Decision Process (MDP)
\((\mathcal{S}, \mathcal{A}, P, r, \gamma)\), where 
\(\mathcal{S}\) denotes the state space,
\(\mathcal{A}\) the action space,
\(P(s'| s,a)\) the transition kernel, 
and \(\gamma \in [0,1)\) the discount factor.
The reward function is a (integrable) random variable
$
r : (\Omega, \mathcal{F}, \mathbb{P}) \to (\mathbb{R}, \mathcal{B}(\mathbb{R})),
$
so that for each state--action pair \((s,a)\), the reward \(r(s,a)\) induces a probability distribution over
\((\mathbb{R}, \mathcal{B}(\mathbb{R}))\).
Here \(\Omega\) is the sample space, \(\mathcal F\) is a \(\sigma\)-algebra of events, \(\mathbb P\) is the probability measure, and \(\mathcal B(\mathbb R)\) denotes the Borel \(\sigma\)-algebra on \(\mathbb R\).
A policy \(\pi(a | s)\) generates a trajectory
\((s_0,a_0,s_1,a_1,\ldots)\), and the return is the random variable
$
Z^\pi \;=\; \sum_{t=0}^{\infty} \gamma^t\, r(s_t, a_t).
$

\subsection{Maximum Entropy IRL}
Given demonstrations \(\{(s_t, a_t)\}_{t\geq 1}\) collected by an unknown expert policy $\pi^E$, MaxEntIRL \citep{ziebart2008maximum} aims to recover the unknown policy, and the corresponding reward function \(r\) the policy is optimized to. Specifically, we consider the following formulation \citep{ho2016generative}:
\begin{equation} \label{eq:old-irl-objective}
\max_{\pi} \min_{r} \E_{d^\pi}[r(s,a)] - \E_{d^{\pi^E}}[r(s,a)] + \mathcal{H}(\pi) + \psi(r),
\end{equation}

where \(d^\pi(s,a)=(1-\gamma)\sum_{t=0}^{\infty}\gamma^t\Pr(s_t=s)\pi(a|s)\) denotes the discounted state-action occupancy measure induced by \(\pi\), \(\mathcal{H}:=\E_{d^\pi} [-\log \pi(a|s)]\) denotes the entropy, and \(\psi\) is a general convex regularizer. 
This formulation reduces to MaxEntIRL if \(\psi=0\). If \(\psi\) is instantiated as a KL penalty between a reward posterior approximation \(q\) and a prior \(p_0\), it resembles the regularized Bayesian objective used in BIRL-style formulations, while the optimal policy still follows a Boltzmann distribution of the action-values
.

\section{Distributional Inverse Reinforcement Learning Framework}

In our model, we treat the reward as a \textit{distribution} rather than a deterministic function. During optimization, the first two terms in Eq.~\ref{eq:old-irl-objective}, 
\(\mathbb{E}_{d^\pi}[r(s,a)] - \mathbb{E}_{d^{\pi^E}}[r(s,a)]\), enforce \textit{mean dominance}--that is, the learned reward should make the expert no worse than the learned policy in expected return. At optimality, this difference becomes zero, indicating \textit{mean matching} between expert and agent returns. However, if the reward is inherently a distribution, mean matching alone fails to capture the relationship between the expert's return distribution and the agent's in its entirety. This leads to a loss of higher-order information in the reward. To accurately model the full reward distribution, we must impose a \textit{distributional form of dominance} during optimization, ensuring that the entire return distribution is aligned at optimality, not just the mean.
We therefore define an ordering over entire distributions.   
\begin{definition}[First-Order Stochastic Dominance (FSD) \citep{hadar1969rules}]
Let \(X\) and \(Y\) be real-valued integrable random variables with cumulative distribution functions \(F_X\) and \(F_Y\). We say that \(X\) \emph{first-order stochastically dominates} \(Y\), written as \(X \succeq_{\rm FSD} Y\), if
$ F_X(z)\;\le\;F_Y(z), \forall\,z\in\mathbb{R}.$
\end{definition}

The concept of FSD is illustrated in Fig.~\ref{fig:quantile-example}. If we aim for \(X \succeq_{\mathrm{FSD}} Y\), then the shaded region indicates a violation of this condition. FSD has an equivalent definition relating to utility functions, which further implies mean dominance. 
\begin{proposition}[Theorem 1-2~\citep{hadar1969rules}]
For real-valued \(X\) and \(Y\), the following are equivalent:
\begin{enumerate}
  \item \(F_X(z)\le F_Y(z)\) for all \(z\in\mathbb{R}\).

  \item \(\displaystyle \mathbb{E}[\,u(X)\,]\;\ge\;\mathbb{E}[\,u(Y)\,]\) for every non-decreasing utility function \(u:\mathbb{R}\to\mathbb{R}\).
\end{enumerate}

\end{proposition}

\begin{corollary}[Mean Dominance]
If \(X\succeq_{\rm FSD} Y\), it follows that
$\mathbb{E}[X]\;\ge\;\mathbb{E}[Y],$
as the identity utility \(u(x)=x\) is non-decreasing.
\end{corollary}

We model the reward as a conditional distribution, \(r_t \sim q_\phi(\cdot | s_t, a_t)\), and define the random return for a trajectory \((s_0, a_0, \dots)\) sampled from policy \(\pi\) as 
\(Z^{\pi} = \sum_{t=0}^\infty \gamma^t r_t\). 
We now introduce the distributional counterpart to Eq.~\ref{eq:old-irl-objective}, the objective for distributional IRL, expressed as
\begin{align}\label{eq:distirl_obj}
    \max_{\pi} \min_{r} \mathcal{L}(\pi, r) :=
    \max_{\pi} \min_{r}
        &\int_{-\infty}^\infty \!\!\!\![F_{Z^E}(z) - F_{Z^\pi}(z)]_+ dz \notag \\
    &+ \mathcal{H}(\pi) + \psi(r),
\end{align}
where $Z^E$ is the return distribution of the expert policy.
\subsection{Learning Reward through Stochastic Dominance}
From Eq.~\ref{eq:distirl_obj}, the objective of the reward function is 
\begin{align*}
   &\min_r \left\{\mathcal{L}_{\text{FSD}}(\pi,r) + \psi(r)\right\}
   =\\
   &
   \min_r \left\{\int_{-\infty}^\infty\!\!\!\! [F_{Z^E}(z) - F_{Z^\pi}(z)]_+ dz + \psi(r)\right\}.
\end{align*}

This objective minimizes the violation of FSD, drawing inspiration from the Kolmogorov-Smirnov (K-S) test~\citep{massey1951kolmogorov}.
To model the reward distribution in a principled manner, we treat
$\mathcal{L}_{\mathrm{FSD}}(\pi, r)$ as an \emph{energy function} that scores how compatible a proposed reward $r$ is with the expert demonstrations.
In particular, we define a likelihood function over the expert demonstrations \(\mathcal{D}\) using the Energy-Based Model (EBM) formulation~\citep{lecun2006tutorial}: \(p(\mathcal{D} | r) \propto \exp\left(-\mathcal{L}_{\text{FSD}}(\pi, r)\right),\)
so that reward functions that yield small FSD violations are exponentially more likely under the expert data.  
This construction is natural here because FSD does not provide an explicit probabilistic model, but \emph{does} provide a calibrated energy landscape that reflects goodness-of-fit.
A more detailed discussion can be found in Appendix~\ref{app:fsd-energy}.

We also introduce a \textit{prior distribution} \(p_0(r)\), which reflects our initial belief before observing any data. The goal is to infer the posterior distribution \(p(r | \mathcal{D})\) using Bayes' rule. As direct inference under the EBM formulation is generally intractable, we adopt the variational inference framework~\citep{blei2017variational} by introducing a \textit{variational reward distribution} \(q_\phi(r | s, a)\), parameterized by \(\phi\), to approximate the posterior and optimize the \textit{evidence lower bound (ELBO)}:
\begin{equation*}
    \mathbb{E}_{q_\phi(r | s, a)}\left[\log p(\mathcal{D} | r)\right] - \mathrm{KL}\left(q_\phi(r | s, a) \,\|\, p_0(r)\right).
\end{equation*}
Substituting the energy-based likelihood into it yields:
\begin{equation} \label{eq:reward-update}
\mathcal{L}_r(\phi) :=\mathbb{E}_{q_\phi(r | s, a)}\left[ \mathcal{L}_{\text{FSD}}(\pi, r) \right] + \mathrm{KL}\left(q_\phi(r) \| p_0\right).
\end{equation}

\noindent Notice the natural relationship between KL and \(\psi\). Formally, we learn the reward distribution by solving Eq.~\ref{eq:reward-update}. To compute the gradient of the first term, we apply the Inverse Transform Sampling technique \citep{devroye2006nonuniform}. We use the empirical quantile to approximate the quantile of the return. Specifically, using the change of variable formula and the relation between CDF and quantile, we have
   $\int_{-\infty}^\infty [F_{Z^E}(z) - F_{Z^\pi}(z)]_+ dz = \int_0^1 \bigl[F_{Z^\pi}^{-1}(v)- F_{Z^E}^{-1}(v)\bigr]_+ \,\mathrm{d} v.$
We provide a short proof of the above relation in Appendix~\ref{prop_A.1}.
To approximate \( F_{Z^\pi}^{-1}\), we draw \(N\) samples \(\{z_n\}\) by Monte Carlo sampling \(z_n = \sum_{t=0}^\infty \gamma^t r_t, r_t\sim q_{\phi}(\cdot| s_t, a_t)\),
and form the empirical quantile using sorted order statistics \(z_{(1)}\le \cdots \le z_{(N)}\), with \(F_{Z^\pi}^{-1}(k/N)\approx z_{(k)}\). 
As a result, minimizing \(\mathcal{L}_r(\phi)\) generalizes the usual IRL objective of matching expected returns by aligning higher-order moments.

\subsection{Risk-aware Policy Learning} \label{sec:risk-aware-policy-learning}
Once the inner minimization over \(r\) yields a fixed reward distribution, the policy, parameterized by \(\varphi\), is updated by maximizing the following objective:
\begin{equation} \label{eq:policy-learning}
\mathcal{L}_\pi(\varphi)= \int_{0}^1 [F_{Z^{\pi_{\varphi}}}^{-1}(v) - F_{Z^E}^{-1}(v)]_+ dv + \mathcal{H}(\pi_{\varphi}).
\end{equation}

\noindent Define \(\mathcal{I}(v) := \mathbbm{1}_{F_{Z^{\pi_{\varphi}}}^{-1}(v) \geq F_{Z^E}^{-1}(v)}\). Fig.~\ref{fig:quantile-example} shows that \(\mathcal{I}(v)\) takes the value 1 in regions where FSD is violated (shaded area), and 0 otherwise. We then rewrite the objective in Eq.~\ref{eq:policy-learning} as 
\begin{equation}\label{eq:iv}
    \int_{0}^1 \left( F_{Z^{\pi_{\varphi}}}^{-1}(v) - F_{Z^E}^{-1}(v) \right) \mathcal{I}(v) dv + \mathcal{H}(\pi_{\varphi}).
\end{equation}

\noindent Note that the indicator function \(\mathcal{I}\) depends on the current policy, the expert policy, and the quantile level \(v\). Conceptually, \(\mathcal{I}\) assigns weight only to regions of the return distribution where FSD is violated. The policy now aims to increase these FSD violations---encouraging the agent to obtain higher return samples in those regions. This leads to a maximization scheme that is inherently risk-aware, as it requires reasoning over the full return distribution rather than just its expectation.

Unfortunately, directly optimizing Eq.~\ref{eq:policy-learning} is intractable, as the indicator function \(\mathcal{I}\) is not observable during training. To address this, we take a broader perspective on risk-aware policy learning and propose replacing \(\mathcal{I}(v)\) with a risk measure that retains the goal of encouraging risk-sensitive behavior while yielding a tractable objective. Furthermore, we show that the resulting surrogate objective provides a weaker form of optimality, but under certain conditions, it can theoretically achieve the same optimum as Eq.~\ref{eq:policy-learning}. To present our new objective, we need a few essential concepts.
\begin{definition}[Distortion function]
    A distortion function \(\xi\) is a non-decreasing function \(\xi:[0,1]\rightarrow [0,1]\) such that \(\xi(0)=0, \xi(1)=1\).
\end{definition}
\begin{definition}[Distortion Risk Measure (DRM) \citep{dhaene2012remarks}] \label{eq:drm}
    For an integrable random variable $X$, and a distortion function \(\xi\), a Distortion Risk Measure \(M_{\xi}\) is defined as $M_{\xi}(X) = \int_0^1 F_X^{-1}(v) d\tilde{\xi}(v)$ where \(\tilde{\xi} = 1-\xi(1-v) \geq 0\) is the dual distortion function.
\end{definition}
Common examples of DRMs and distortion functions are listed in Table~\ref{table:spectral_risk}. These measures offer various ways to quantify risk based on the return distribution. When \(d\tilde{\xi}\) admits a density, the resulting DRM is often called a spectral risk measure; in the paper we use the broader term DRM. Intuitively, when \(\tilde{\xi}\) is concave, it places greater emphasis on lower returns, thereby encouraging risk-averse behavior.
To induce risk-aware policies using distortion $\xi(v)$, we maximize the DRM defined in Def.~\ref{eq:drm}. In all main experiments, \(\xi\) is fixed within each run; unless otherwise stated, we use CVaR with risk parameter \(0.05\), with additional DRM choices studied in Appendix~\ref{appendix:drm_choices}.

Building on the above definitions, we propose replacing \(\mathcal{I}(v)\) with \(\tilde{\xi}(v)\) in Eq.~\ref{eq:iv}, resulting in:
$\max_{\varphi} \int_{0}^1 \left( F_{Z^\pi}^{-1}(v) - F_{Z^E}^{-1}(v) \right) d\tilde{\xi}(v) + \mathcal{H}(\pi) = \max_{\varphi} \int_{0}^1  F_{Z^\pi}^{-1}(v) d\tilde{\xi}(v) + \mathcal{H}(\pi).
$
The equality holds because the expert policy does not depend on \(\varphi\). We denote the final objective as
\begin{equation}\label{eq:policy-update}
    \mathcal{L}_{\pi}(\varphi) 
    = \int_{0}^1  F_{Z^{\pi_\varphi}}^{-1}(v) d\tilde{\xi}(v) + \mathcal{H}(\pi_\varphi),
\end{equation}

\noindent where \(M_{\xi}\) is a chosen DRM with a distortion function \(\xi\).

\textit{\underline{Relation to Eq.~\ref{eq:policy-learning}.}}
Additionally, we know that \(X \succeq_{\rm FSD} Y\Rightarrow M_{\xi}(X) \geq M_{\xi}(Y)\) \citep{sereda2010distortion}. A natural question is which conditions are sufficient for FSD. We observe that the converse requires stronger conditions.
\begin{restatable}{proposition}{sufficiencylemma}
    \(M_{\xi}(X) \geq M_{\xi}(Y)\) for every distortion function \(\xi\) implies \(X \succeq_{\rm FSD} Y\).
\end{restatable}
The proof is straightforward by observing that \(M_{\xi}(X) - M_{\xi}(Y) = \int_0^1 (F_X^{-1}(v) - F_Y^{-1}(v)) d\tilde{\xi}(v)\) and the fact that \(\tilde{\xi}(v) \geq 0\). We present a short proof in Appendix~\ref{proof:prop_4.6}. This implies that if we solve 
\(
\max_{\pi_\varphi} \int_0^1 \left( F_{Z^{\pi_{\varphi}}}^{-1}(v) - F_{E}^{-1}(v) \right) d\tilde{\xi}(v) + \mathcal{H}(\pi_{\varphi})
\) for every distortion function, we obtain the solution to Eq.~\ref{eq:policy-learning}. However, since optimizing over all utility conditions is intractable, our proposed objective serves as an approximation using a specific DRM. Nonetheless, under the conditions of the proposition, this surrogate objective can theoretically achieve the same optimality as Eq.~\ref{eq:policy-learning}.

\subsection{Practical Algorithm}
\begin{center}
\begin{algorithm}[h]
\caption{A DistIRL method with FSD objective}
\label{alg:DistIRL}
\KwIn{Expert data $\mathcal{D} = \{(s_t^E, a_t^E)\}$, prior $p_0(r)$, risk measure $\xi$, step sizes $\eta^\theta, \eta^\varphi, \eta^\phi$}
\KwOut{Reward distribution $q_\phi(r | s, a)$; policy $\pi_\varphi(a | s)$}
\BlankLine

Initialize parameters of reward network $\phi$, policy $\varphi$, and critic $\theta$\;

\For{$k = 1$ \KwTo $K$}{

    Sample a mini-batch $\{(s_t^E, a_t^E)\}$ from $\mathcal{D}$\;

    \ForEach{$(s_t^E, a_t^E)$ in mini-batch}{
        For each $s_t^E$, sample $a_t \sim \pi_\varphi(\cdot| s_t^E), r_t \sim q_\phi(\cdot|s_t^E, a_t), r_t^E \sim q_\phi(\cdot|s_t^E, a_t^E)$\;
    }

    Compute return samples $Z^{\pi_k},Z^E$\;
    Critic update via quantile regression (Eq.~\ref{eq:huber-loss}):
    \(
    \theta_{k+1} \leftarrow \theta_k - \eta^\theta \nabla\mathcal{L}_{QR}(\theta_k)
    \)\;
    
    Policy update with distortion risk measure (Eq.~\ref{eq:policy-update}):
    \(
    \varphi_{k+1} \leftarrow \varphi_k + \eta ^\varphi\nabla\mathcal{L}_{\pi}(\varphi_k)
    \)\;
    
    Reward distribution update via FSD loss (Eq.~\ref{eq:reward-update}):
    \(
    \phi_{k+1} \leftarrow \phi_k - \eta^\phi \nabla\mathcal{L}_r(\phi_k).
    \)
}

\end{algorithm}
\end{center}

To enable tractable and expressive modeling of reward uncertainty, we parameterize the reward distribution $q_\phi(r|s,a)$, for example, using Azzalini's skew-normal distribution \citep{azzalini1996multivariate}: 
\(q_\phi(r|s,a) = \mathcal{SN}(\mu_\phi(s,a), \sigma^2_\phi(s,a); \alpha_\phi(s,a)),\)
where the mean $\mu_\phi(s,a)$, standard deviation $\sigma_\phi(s,a)$, and skew parameter \(\alpha_\phi(s,a)\) are outputs of a neural network with parameters $\phi$. This choice allows for efficient sampling and computing regularization when using a standard normal prior. During training, for each state-action pair, we sample rewards $r_t \sim q_\phi(\cdot|s_t, a_t)$ to construct return samples for both the expert and the current policy.

Note that the choice of prior depends heavily on the task domain and the type of variability we expect in the reward signal. This prior sensitivity is shared with Bayesian IRL methods: an informative prior can help when domain knowledge is reliable, while a poor prior can bias the learned reward distribution. For example, skew-normal distributions can capture asymmetric reward uncertainty in tasks with systematic biases (e.g., contact-rich manipulation), whereas heavy-tailed priors may be more suitable when outliers or rare but significant events dominate the return structure. In contrast, the broader statistical learning community often defaults to Gaussian priors, primarily because of their analytical tractability, conjugacy with many likelihood models, and well-understood concentration properties. That said, DistIRL does not rely on a fixed distributional assumption. Any parameterized distribution ${p_\theta}$ whose log-density or quantile function is differentiable in $\theta$ is compatible with our framework, since the algorithm only requires gradients.

To estimate the DRM $M_{\xi}(Z^{\pi})$ for the policy, we follow an offline approach: we use states $s_t$ drawn from the expert demonstration dataset, but sample actions $a_t^\pi \sim \pi_\theta(\cdot|s_t)$ from the current policy, and a reward $r_t \sim q_\phi(\cdot|s_t, a_t^\pi)$. Then we compute the return $Z^{\pi}$ by taking the sum. 
For policy update, we first learn the critic by Off-policy Evaluation (OPE) \citep{sutton1998reinforcement} on \((s_t, a_t, r_t, s_{t+1}, a_{t+1}^\pi)\) where we use Quantile Regression with the Quantile Huber loss \(\mathcal{L}_{QR}\) as in Eq.~\ref{eq:huber-loss}. We then update the risk-aware policy by solving $\min_{\pi} \mathrm{KL}\left( \pi(\cdot | s) \,\big\|\, \frac{1}{Z} \exp\left\{ M_{\xi}(Z^{\pi}(\cdot | s)) \right\} \right)$, which corresponds to the KKT solution to Eq.~\ref{eq:policy-learning}, as originally introduced by~\citet{ziebart2008maximum}. We summarize the full procedure in Alg.~\ref{alg:DistIRL}. 

\section{Theoretical Results}
In this section, we provide a theoretical analysis of the algorithm proposed above. The analysis fixes a distortion function \(\xi\) throughout the run and studies the idealized update in which the corresponding DRM policy-improvement step is solved exactly. For a reward parameter \(\phi\) and policy \(\pi\), let \(Q^\xi_{\phi,\pi}\) denote the risk-sensitive critic induced by the nested DRM Bellman operator. Let \(\pi^\star_\phi\) denote the DRM-optimal policy for the reward distribution \(q_\phi\), and define the critic tracking error
\[
E_k = \big\|Q^{\xi}_{\phi_k,\pi_k} - Q^{\xi}_{\phi_k,\pi^\star_{\phi_k}}\big\|_\infty .
\]
The stepsize exponent \(\sigma\in(0,1)\) controls the reward-update stepsize schedule. First, we introduce several regularity assumptions, the necessity of which is detailed in the appendix~\ref{sec:appendix_convergence}. 

\begin{restatable}{assumption}{boundedapp}
\label{ass:bounded-app}
There exists $R_{\max}<\infty$ such that
\begin{equation}
\label{eq:bounded-reward}
|q_\phi(s,a)| \;\le\; R_{\max}
\quad\text{almost surely for all }(s,a,\phi).
\end{equation}
\end{restatable}

\begin{restatable}{assumption}{lawlipapp}
\label{ass:lawlip-app}
For every $(s,a)$ and all $\phi_1,\phi_2\in\mathbb R^d$, the reward laws satisfy
\begin{equation}
\label{eq:reward-wasserstein}
W_\infty\!\big(q_{\phi_1}(\cdot| s,a),\,q_{\phi_2}(\cdot| s,a)\big)
\;\le\; L_R \,\|\phi_1 - \phi_2\|,
\end{equation}
where $W_\infty$ denotes the Wasserstein infinity distance. Equivalently, one can couple $q_{\phi_1}(s,a)$ and $q_{\phi_2}(s,a)$ such that
$
|q_{\phi_1}(s,a)-q_{\phi_2}(s,a)|
\;\le\; L_R\,\|\phi_1 - \phi_2\|
\quad\text{almost surely}.
$
\end{restatable}

We use the following assumption on a given DRM. Standard normalized DRMs satisfy these properties. 
\begin{restatable}{assumption}{drmapp}
\label{ass:drm-app}
For each state-action pair $s\in\mathcal S, a\in \mathcal{A}$, the one-step distortion risk measure
$M_\xi(\cdot| s,a)$ is
\begin{enumerate}
    \item \emph{monotone}: $X\le Y$ a.s.\ implies $M_\xi(X| s,a)\le M_\xi(Y| s,a)$;
    \item \emph{translation-equivariant}: $M_\xi(X+c| s,a) = M_\xi(X| s,a)+c$ for all $c\in\mathbb R$;
    \item \emph{$1$-Lipschitz in $\|\cdot\|_\infty$}: for all bounded random variables $X,Y$,
    $\big|M_\xi(X| s,a) - M_\xi(Y| s,a)\big|
    \;\le\; \|X-Y\|_\infty.$
\end{enumerate}
\end{restatable}
First, we show that the critic under a given DRM converges in the average sense:

\begin{restatable}{theorem}{qconvnew}
\label{thm:q-conv-new}
Assume Assumptions~\ref{ass:bounded-app}-\ref{ass:drm-app} hold. 
Let $E_k = \big\|Q^{\xi}_{\phi_k,\pi_k} - Q^{\xi}_{\phi_k,\pi^\star_{\phi_k}}\big\|_\infty$.
Assume the reward update
satisfies Assumption~\ref{ass:smooth-grad}, with stepsizes
$\eta_k = \eta= \eta_0 K^{-\sigma}$, $\eta_0>0$, and $\sigma\in(0,1)$.
Then running the DistIRL algorithm $K$ steps, we have $\textstyle{\frac{1}{K}\sum_{k=1}^K E_k
\;=\;
\mathcal O(K^{-1})
\;+\;
\mathcal O(K^{-\sigma}).}$

\end{restatable}
This gives a corresponding policy bound:

\begin{restatable}{theorem}{policyconvnew}
\label{thm:policy-conv-new}
For each $k$, define the learned and DRM-optimal policies induced by the
current $Q$-functions:
\begin{equation*}
\label{eq:pi-plus-def}
\begin{aligned}
\pi_k(\cdot| s)
\!\propto \!\exp \big(Q^{\xi}_{\phi_k,\pi_k}\!\!(s,\cdot)\big),
\pi^{\star}_{\phi_k}(\cdot| s)
\!\propto\! \exp  \big(Q^{\xi}_{\phi_k,\pi^\star_{\phi_k}}\!\!(s,\cdot)\big).
\end{aligned}
\end{equation*}
Then, running the DistIRL algorithm $K$ steps, we have
\begin{equation*}
\label{eq:policy-gap-avg}
\textstyle\frac{1}{K}\sum_{k=1}^K
\|\log \pi_k - \log \pi^{\star}_{\phi_k}\|_\infty
= \mathcal O(K^{-1}) + \mathcal O(K^{-\sigma}).
\end{equation*}
\end{restatable}
Finally, we obtain a rate of convergence towards a first-order stationary point:
\begin{restatable}{theorem}{firstorder}
\label{thm:first-order}
Suppose Assumptions~\ref{ass:bounded-app},
\ref{ass:lawlip-app}, \ref{ass:smooth-grad},
and~\ref{ass:grad-bias} hold.
Let $\eta_k = \eta_0 k^{-\sigma}$ with $\eta_0>0$ and
$\sigma\in(0,1)$, and assume $\mathcal{L}_r$ is bounded below on $\Phi$.
Then there exists $C>0$ such that
\begin{equation*}
    \frac{1}{K}\!\!\sum_{k=0}^{K-1}
    \E\|\nabla\!\mathcal{L}_r(\phi_k)\|^2
   \! =\!
    \mathcal{O}\!\!\left( K^{-1} \right) + \mathcal{O}\!\!\left(K^{-\sigma} \right) + \mathcal{O}\!\!\left(K^{-1+\sigma} \right).
    \label{eq:avg-grad-bound-final}
\end{equation*}
\end{restatable}
In particular, picking $\sigma=1/2$, we obtain a $\mathcal{O}(\varepsilon^{-2})$ iteration bound for the algorithm to reach an $\varepsilon$-stationary point in averaged squared gradient norm. This rate characterizes the alternating DistIRL objective under our assumptions.

\section{Experiment}
\subsection{Gridworld}

We begin with a \(5 \times 5\) gridworld environment where the agent is trained to navigate from the starting state \((2,0)\) (left-center) to rewarding goal locations. Two high-reward states are placed at \((0,4)\) (top-right) and \((4,4)\) (bottom-right), with the top-right reward modeled as a stochastic outcome drawn from \(\mathcal{N}(1,1)\). The first column of Fig.~\ref{fig:gridworld} illustrates the ground-truth reward mean and variance.

\begin{figure}[h!]
\small
    \centering
    \includegraphics[width=\linewidth]{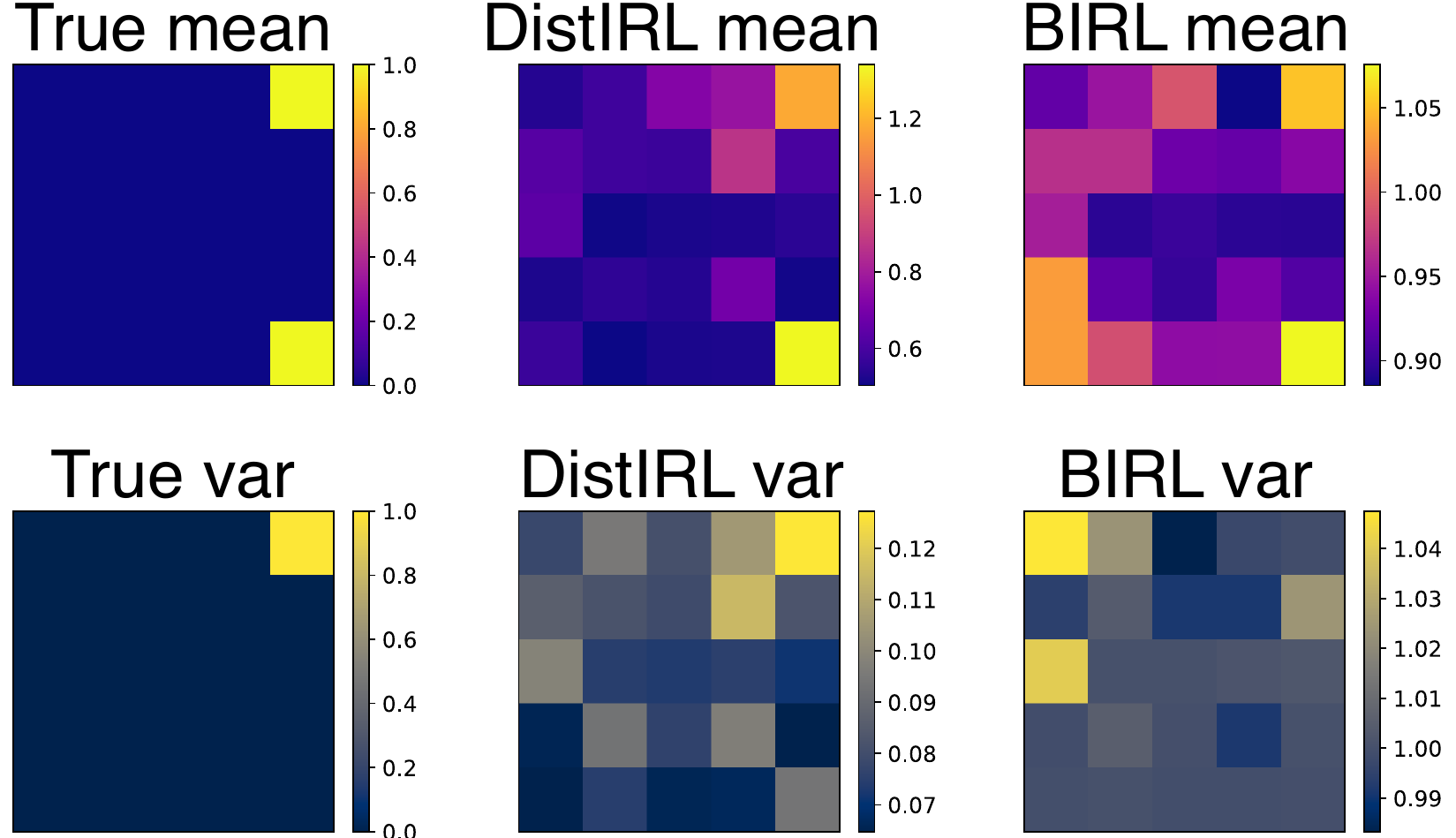}
    \caption{Inferring reward mean and variance in the gridworld example with $10$ demonstrations.}
    \label{fig:gridworld}
\end{figure}

This setup mimics an animal exploring an arena with two reward ports. In such compact environments, animals often display risk-averse behavior, i.e., avoiding locations where rewards have previously failed to appear \citep{mobbs2018foraging, daw2006cortical}. To model this, we collect $10$ trajectories from a risk-averse agent trained under stochastic rewards. In $9$ out of $10$ episodes, the agent chooses the more reliable bottom-right goal.
We then apply our DistIRL method to recover the full reward distribution. As shown in Fig.~\ref{fig:gridworld}, using a symmetric Gaussian reward estimator combined with risk-averse policy learning, our approach not only identifies both high-reward states but also captures the variance at the top-right goal. This highlights the model's ability to infer higher-order moments of the reward from expert demos.

As a baseline, we evaluate Bayesian IRL (BIRL) \citep{chan2021scalable, mandyam2023kernel, bajgar2024walking}. 
BIRL is a widely used framework that assumes a reward distribution but learns it by matching only the mean, without capturing the full distributional structure. We select BIRL because it is the method most comparable to ours in its ability to recover a reward distribution. BIRL reasonably recovers the mean reward but produces spurious high estimates in the lower-left corner. Furthermore, it fails to capture reward variance, emphasizing the need to enforce distance over the full distribution. 
Simply specifying a reward distribution, without integrating distribution-aware learning, fails to capture the true variance of the rewards.

\subsection{Mouse Spontaneous Behavior}

We apply our framework to a neuroscience dataset in which mice freely explore an arena without explicit rewards \citep{markowitz2023spontaneous}. Behavior was recorded using a depth camera, and the raw trajectories were converted into sequences of discrete syllables (e.g., grooming, sniffing). We model these trajectories with an MDP, treating each syllable as a state and the next syllable as the action, yielding ten states and ten actions. In total, we analyzed 159 such state-action sequences. The dataset also includes a time-aligned one-dimensional trace of dopamine fluctuations from the dorsolateral striatum. Prior work \citep{markowitz2023spontaneous} showed that using dopamine as a reward enabled a simulated RL agent to reproduce observed transitions, suggesting IRL should recover a reward pattern resembling dopamine. Since dopamine varies even within the same state-action pair, the prior study used only its mean for simplicity. Here, we compare rewards learned under deterministic vs. distributional assumptions to assess how well they capture both the mean and the full distribution of dopamine signals.

\begin{figure*}[t]
    \centering
    \includegraphics[width=\textwidth]{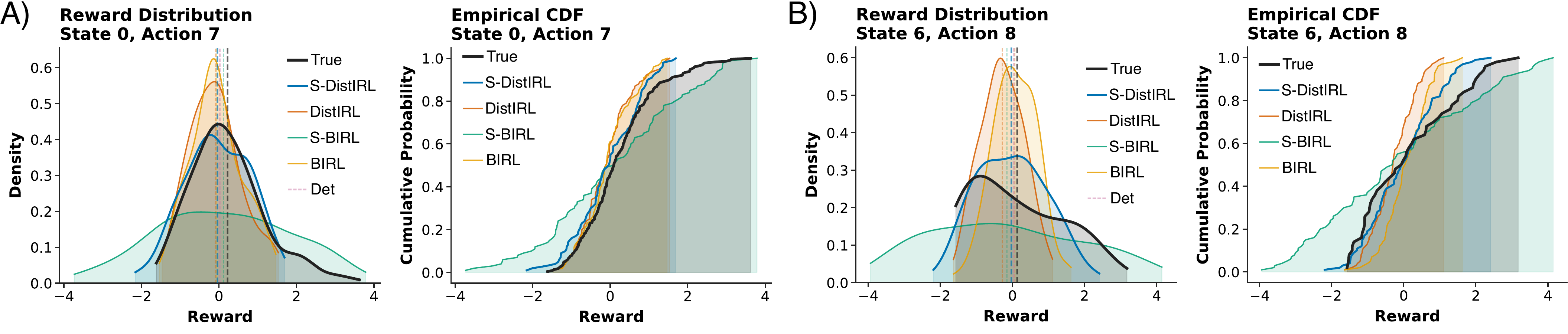}
    \caption{Learned reward distribution versus recorded dopamine signals and their empirical CDFs.}
    \label{fig:da_dist}
\end{figure*}

We use both Azzalini's skew-normal distribution (denoted ``S-'') and the symmetric Gaussian as reward models for both DistIRL and BIRL. Fig.~\ref{fig:da_dist}A) and B) show two example state-action pairs, illustrating the true dopamine fluctuation distribution alongside the estimated reward distributions from four methods. 
The assumption of a parameterized reward distribution is motivated by prior findings in computational neuroscience: dopamine-related reward signals in rodents are well known to exhibit asymmetric, left-skewed variability. For this reason, we chose a skew-normal family, which captures exactly this type of asymmetric structure while remaining interpretable. 
For each case, we display both the probability density function and the CDF, along with the corresponding means. Deterministic rewards (Det) are shown as pink dashed lines in the density plots. Among all methods, S-DistIRL most accurately recovers the shape of the dopamine distribution, which is often right-skewed and multimodal. Its estimated mean also closely matches both the true mean and the deterministic estimate.

\begin{figure}
    \includegraphics[width=1.0\linewidth]{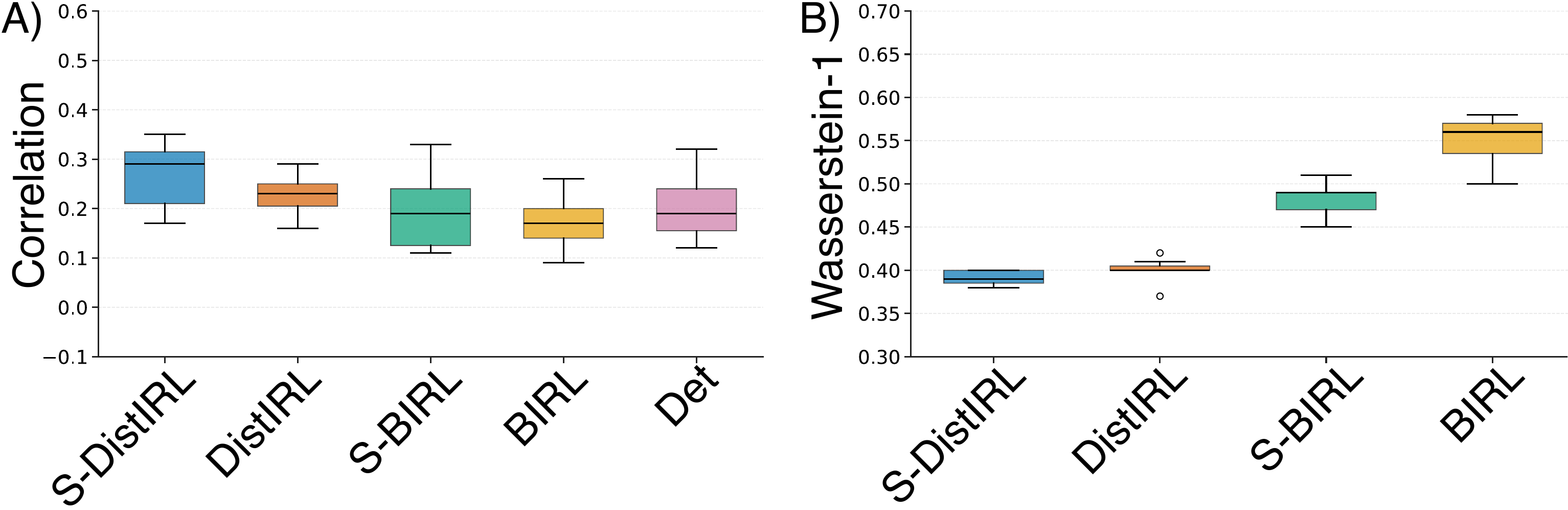}
    \caption{Left: Pearson correlation of the reward mean and dopamine level. Right: W-1 loss between learned distribution and dopamine level.}
    \label{fig:da_skew_compare}
\end{figure}

We also quantify the similarity between estimated rewards and actual dopamine distributions. In Fig.~\ref{fig:da_skew_compare}A), we report the correlation between the mean of dopamine fluctuations and the mean of the estimated reward across all mice and trajectories. Deterministic reward models yield moderate correlation, while DistIRL improves upon this, with S-DistIRL achieving the highest correlation overall. This finding indicates that incorporating full reward distributions, using suitable skewed distributional models, is essential for IRL to capture biologically meaningful reward signals. Fig.~\ref{fig:da_skew_compare}B) shows that, compared to BIRL, S-DistIRL also achieves a lower Wasserstein-1 distance between the estimated reward distribution and the actual dopamine distribution, indicating better alignment of the shape. Taken together, both qualitative examples and quantitative metrics support that modeling skewed reward distributions significantly enhances the ability to track dopamine fluctuations.

This result suggests that reward structure can be inferred directly from behavior data under appropriate modeling assumptions. While it is known that dopamine neurons encode reward-related signals \citep{schultz1997neural, markowitz2023spontaneous}, our experiment shows not only a nontrivial correlation between the inferred and measured mean rewards (with a correlation around 0.3), but also that the full reward distribution recovered from behavior reasonably resembles the distribution of dopamine fluctuations. This suggests that detailed features of neuromodulatory signals, such as the variability in dopamine release, can be decoded from behavior alone, highlighting the potential of inverse modeling to uncover internal motivational states and their neural substrates.
\subsection{MuJoCo Benchmarks}

\noindent\textbf{Risk-sensitive D4RL.} 
In earlier experiments, we applied DistIRL to discrete state-action MDPs and compared it with BIRL. Here we extend the study to continuous MDPs to demonstrate DistIRL's scalability and generalizability. We evaluate our method on Risk-sensitive D4RL benchmarks, following the reward formulations introduced in recent robustness studies \citep{urpi2021risk}. Specifically, the reward functions incorporate stochastic penalties triggered by safety-related conditions:\\
(1) \textbf{Half-Cheetah:} 
\(
R_t(s,a) = \bar{r}_t(s,a) - 70 \mathbb{I}_{\nu > \bar{\nu}} \cdot \mathcal{B}_{0.1},
\)
where $\bar{r}_t(s,a)$ is the environment reward, $\nu$ is the forward velocity, and $\bar{\nu}$ is a velocity threshold ($\bar{\nu}=4$ for the medium variant and $\bar{\nu}=10$ for the easy variant). This penalty models rare but catastrophic robot failures at high speed.\\
(2) \textbf{Walker2D/Hopper:} 
\(
R_t(s,a) = \bar{r}_t(s,a) - p \mathbb{I}_{|\theta| > \bar{\theta}} \cdot \mathcal{B}_{0.1},
\)
where $\bar{r}_t(s,a)$ is the environment reward, $\theta$ is the pitch angle, $\bar{\theta}$ is a task-dependent threshold ($0.5$ for Walker2D-M/E and $0.1$ for Hopper-M/E), and $p$ is the penalty magnitude ($30$ for Walker2D and $50$ for Hopper). 

We train expert agents on these stochastic reward formulations using Risk-averse Distributional SAC, a variant of DSAC \citep{duan2021distributional} with a CVaR objective, and collect $10$ demonstration trajectories. We then evaluate DistIRL against several state-of-the-art baselines. Results are averaged over 5 random seeds. We use a standard normal as the prior due to its general applicability when the underlying true reward distribution is unknown.

Table~\ref{table:stochastic_mujoco} shows that our method consistently outperforms other \textbf{offline} IRL baselines under stochastic reward settings. For reward parameterization, we use a Gaussian distribution (denoted as \textbf{DistIRL}) and a quantile-function parameterization (denoted as \textbf{DistIRL-qtr}, short for QuanTile Reward). Popular online methods such as GAIL \citep{ho2016generative} are not directly applicable in this offline setting. \textbf{Offline ML-IRL} \citep{zeng2023demonstrations} is a model-based MaxEntIRL method that relies on a separately trained transition model using additional non-expert data. Its poor performance here is expected: the transition model was pretrained under risk-neutral rewards and does not align with the new expert data generated under risk-sensitive objectives, leading to severe distribution mismatch. \textbf{ValueDICE} \citep{kostrikov2019imitation}, a model-free offline MaxEntIRL baseline, also underperforms since it optimizes with respect to expected risk-neutral returns, while our experts follow risk-averse behavior. \textbf{Behavior Cloning (BC)} achieves moderately strong results, as it simply mimics the demonstrated actions without explicitly optimizing for either risk-neutral or risk-sensitive objectives. However, its performance is limited as the model overfits.

\begin{table}[ht]
\footnotesize
\setlength{\tabcolsep}{4pt}
\centering
\caption{D4RL performance averaged over 5 seeds.}
\vspace{-0.05in}
\label{table:stochastic_mujoco}
\begin{tabular}{lccc}
\toprule
Method & HalfCheetah & Hopper & Walker2d \\ \midrule
DistIRL (ours) 
  & $\mathbf{3469 \pm 59}$ 
  & $\mathbf{886 \pm 1}$ 
  & $\mathbf{1526 \pm 148}$ \\
DistIRL-qrt (ours) 
  & $3294 \pm 172$ 
  & $747 \pm 79$ 
  & $1211 \pm 182$ \\
Offline ML-IRL 
  & $826 \pm 231$ 
  & $192 \pm 56$ 
  & $240 \pm 50$ \\
ValueDICE 
  & $1259 \pm 78$ 
  & $260 \pm 10$ 
  & $798 \pm 311$ \\
BC 
  & $2828 \pm 281$ 
  & $346 \pm 1$ 
  & $1321 \pm 26$ \\ \hline
Expert 
  & $3540 \pm 44$ 
  & $892 \pm 3$ 
  & $1478 \pm 200$ \\
\bottomrule
\end{tabular}
\vspace{-0.1in}
\end{table}

\begin{figure}
\small
    \centering
    \includegraphics[width=0.8\linewidth]{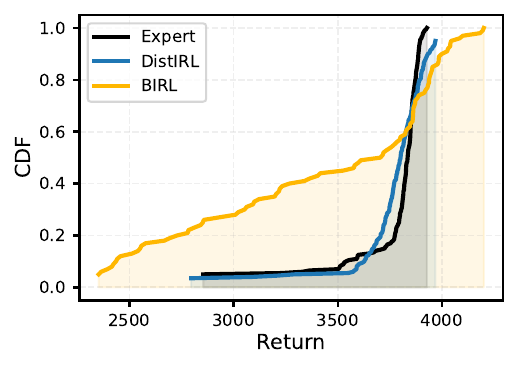}
    \vspace{-0.1in}
    \caption{Return distributions comparison in HalfCheetah.}
    \label{fig:halfcheetah_cdfs}
    \vspace{-0.2in}
\end{figure}

To further validate the fidelity of our inferred return distributions from DistIRL and compare with the BIRL framework that only matches the mean, we collect $200$ trajectories, sample the learned return distribution for each learned policy, and plot against the expert's return distribution in Fig.~\ref{fig:halfcheetah_cdfs}. This shows that DistIRL's reward and policy model better align with the expert. We also report a Pearson correlation coefficient of $0.92$ between the mean estimated by DistIRL and the mean of the true return. This indicates strong agreement and demonstrates that our inferred reward is an accurate proxy for the true reward model. A further examination of the return distribution and its higher-order moments can be found in Appendix~\ref{appendix:match_return_distribution}.

\vspace{-0.1in}
Additionally, the competitive results of quantile-based reward parameterization open the opportunity to use a broad range of parametric families, including diffusion models, and we leave this direction as a future extension.

\noindent\textbf{Risk-neutral D4RL.} 
We also test our algorithm in conventional deterministic reward settings using D4RL's \texttt{medium-expert} trajectories \citep{fu2020d4rl}. Table~\ref{table:det_mujoco} shows that our method achieves competitive or superior performance even without tailoring to deterministic assumptions, underscoring the generality of DistIRL. We want to emphasize that Offline ML-IRL requires additional data\footnote{For HalfCheetah, with the same amount of data as Offline ML-IRL, DistIRL can reach $11239 \pm 539$.}.

\textbf{Ablation studies.} We evaluate the contribution of different design choices by ablating our model under the HalfCheetah setting with right-skewed normal (\(\mathcal{SN}_\eta, \eta>0\)) stochastic rewards and risk-averse expert policy, indicating the expert prefers conservative actions that yield more consistent rewards. Variants include:  
\textbf{Dis/Det}: Distributional or Deterministic rewards;  
\textbf{QR/TD}: Quantile Regression or TD-based critic;  
\textbf{FSD/Mean}: FSD loss or Mean matching.
As shown in Table~\ref{table:ablation-study}, which scales the performance between worst and best, using distributional rewards with FSD loss significantly outperforms mean-matching alternatives. Additionally, deterministic TD-learning with mean-matching (\textbf{Det-TD-Mean}) underperforms in learning risk-averse policies due to a lack of distributional supervision. This confirms the effectiveness of FSD-based reward learning and risk-sensitive policy optimization. Note that the BIRL framework aligns with our \textbf{Dis-TD-Mean} configuration; RIZE \citep{karimirize} aligns with \textbf{Det-Qt-Mean}, which performs the worst; \textbf{Det-TD-Mean} aligns with ValueDice but with an explicit reward estimation. Thus, in this ablation study, we treat them as a specific setting within DistIRL when benchmarking against other approaches.

\vspace{-0.1in}
Additionally, we conduct ablation studies on the choice of DRM in Appendix~\ref{appendix:drm_choices}, showing that DistIRL is not sensitive to a specific DRM as long as the chosen risk measure does not deviate too far from the underlying risk preference of the expert data. We also conduct experiments on the number of trajectories for the risk-sensitive D4RL dataset in Appendix~\ref{appendix:number_of_trajs}, which show that DistIRL is sufficiently robust in a low-data regime, indicating that our approach is computationally attractive.  

\vspace{-0.05in}
\begin{table}[ht]
\centering
\small
\setlength{\tabcolsep}{4pt}  
\caption{Performance on deterministic reward settings (D4RL).}
\vspace{-0.05in}
\label{table:det_mujoco}
\begin{tabular}{lccc}
\toprule
Method & HalfCheetah & Hopper & Walker2d \\ \midrule
DistIRL (Ours)
  & $7779 \pm 228$
  & $\mathbf{3411 \pm 42}$
  & $\mathbf{4570 \pm 305}$ \\
Off. ML-IRL
  & $\mathbf{11231 \pm 585}$
  & $3347 \pm 238$
  & $4201 \pm 638$ \\
ValueDICE
  & $4935 \pm 2836$
  & $3073 \pm 539$
  & $3191 \pm 1888$ \\
BC
  & $623 \pm 56$
  & $3236 \pm 46$
  & $2822 \pm 979$ \\ \hline
Expert
  & $12175 \pm 91$
  & $3512 \pm 22$
  & $5384 \pm 52$ \\
\bottomrule
\end{tabular}
\end{table}
\vspace{-0.9\baselineskip}

\begin{table}[ht]
\small
  \caption{Ablation study on model setting. Performance scaled.}
\vspace{-0.05in}
  \label{table:ablation-study}
  \centering
  \begin{tabular}{@{}ccc@{}}
    \toprule
    Dis-Qt-Mean & Det-Qt-Mean & DistIRL (Ours) \\ \midrule
    $0.22 \pm 0.02$ & $0.00 \pm 0.01$ & $\mathbf{1.00 \pm 0.02}$ \\ \midrule
    Dis-TD-FSD & Dis-TD-Mean & Det-TD-Mean \\ \midrule
    $0.67 \pm 0.31$ & $0.33 \pm 0.01$ & $0.22 \pm 0.00$ \\
    \bottomrule
  \end{tabular}
\end{table}


\vspace{-0.2in}
\section{Conclusion}
\vspace{-0.1in}
We introduce a distributional framework for inverse reinforcement learning that jointly models reward uncertainty and return distributions. Our method enables risk-aware policy learning and accurate inference of high-order structure in demonstrations. We validate the framework on stochastic control tasks, deterministic settings, and real neural datasets, demonstrating state-of-the-art performance and strong generalization across domains. Like other IRL methods, DistIRL does not claim unique recovery of a ground-truth reward from demonstrations alone; it recovers a compatible reward distribution under the chosen prior, variational family, and FSD-based inductive bias. Future work should expand real-world validation, learn or adapt the DRM from demonstrations, study targeted finite-moment relaxations when only selected moments are scientifically important, and model correlations among state-action reward distributions.


\section*{Impact Statement}

IRL enables powerful tools for understanding behavior, with positive applications in neuroscience, animal modeling, and AI alignment. However, it also raises ethical concerns. IRL could be misused in military settings to model or mimic adversarial behavior, or in surveillance contexts to infer personal goals without consent, posing risks to privacy and autonomy. These concerns highlight the need for careful oversight and responsible deployment.

\bibliography{reference}
\bibliographystyle{icml2026}

\newpage
\appendix
\onecolumn
\section{Related work comparison}
\label{tab:lit_review}
\begin{table}[h]
    \small
    \centering
    \caption{Comparison of IRL methods under various settings}
    \begin{tabular}{l|c|c|c|c}
        \toprule
        \textbf{Reference} & 
        \makecell{Model\\reward dist.?} & 
        \makecell{Infer\\risk aware\\policy?} & 
        \makecell{Recover\\reward dist.?} & 
        \makecell{Learn\\return dist.?} \\
        \midrule
        \makecell[l]{\citep{wulfmeier2015maximum, ziebart2008maximum} \\ \citep{garg2021iq,  ni2021f} \\ \citep{zeng2022maximum, zeng2023demonstrations, wei2023bayesian}} 
        & \xmark & \xmark & \xmark & \xmark \\
        \midrule
        \makecell[l]{\citep{ramachandran2007bayesian, choi2011map} \\ \citep {chan2021scalable, lee2022risk}} 
        & \cmark & \xmark & \xmark & \xmark \\
        \midrule
        \citep{karimirize} & \xmark & \xmark & \xmark & \cmark \\
        \midrule
        \citep{singh2018risk, lacotte2019risk} \\ \citep{ cheng2023eliciting} & \xmark & \cmark & \xmark & \xmark \\
        \midrule
        This work & \cmark & \cmark & \cmark & \cmark \\
        \bottomrule
    \end{tabular}
\end{table}

In Table~\ref{tab:lit_review}, we compare DistIRL with existing IRL methods along four key dimensions. The first column, \textit{Model reward distribution}, asks whether a method explicitly represents the reward as a random variable rather than as a fixed deterministic function. For example, Bayesian IRL methods place a prior over reward parameters, thereby modeling uncertainty, but they do not recover the actual shape of the underlying distribution. This is distinct from \textit{Recover reward distribution}, which requires learning the full distribution of rewards themselves, including higher-order statistics such as variance and skewness, rather than just a posterior over parameters.

The third column, \textit{Infer risk-aware policy}, evaluates whether a method incorporates risk measures into policy inference. Methods in this category optimize beyond expected return, often capturing aversion or preference to variability in outcomes. The final column, \textit{Learn return distribution}, indicates whether a method leverages distributional reinforcement learning (DistRL) techniques to estimate the full distribution of returns, rather than only their expectation. Unlike reward distributions, which describe stochasticity at the immediate reward level, return distributions capture the cumulative effect of randomness from rewards, transitions, and policies over trajectories.

As shown in the table, most prior IRL methods either assume deterministic rewards or restrict themselves to expectation-based inference. In contrast, DistIRL simultaneously models stochastic rewards, learns full reward distributions, integrates distributional return estimation, and supports risk-aware policy learning, thereby unifying these capabilities in a principled way.

{
\section{Extended preliminaries}
The state-value and action-value functions under \(\pi\) are defined as
\[
V^\pi(s) = \E\bigl[Z^\pi | s_t=s\bigr], 
\qquad
Q^\pi(s,a) = \E\bigl[Z^\pi | s_t=s,a_t=a\bigr].
\]
They satisfy the Bellman equations
\[
V^\pi(s) = \E_{a\sim\pi,s'\sim P}\left[r(s,a) + \gamma V^\pi(s')\right],
\quad
Q^\pi(s,a) = \E_{s'\sim P}\bigl[\,r(s,a) + \gamma\,\E_{a'\sim\pi}[Q^\pi(s',a')]\bigr].
\]
We also define the \emph{occupancy measure} of $\pi$ as 
\(
d^\pi(s,a)
=
(1-\gamma)\sum_{t=0}^\infty \gamma^t\,\Pr(s_t = s)\,\pi(a| s),
\)
which satisfies $\sum_{s,a}d^\pi(s,a)=1$ and characterizes the long-run state-action visitation probability.
\vspace{-0.5\baselineskip}
\subsection{Distributional RL and Risk-Sensitive Control}
\vspace{-0.5\baselineskip}

Rather than estimating only \(\E[Z^\pi]\), distributional RL models the entire return distribution
that obeys the {\em distributional Bellman operator} \(\mathcal{T}^\pi\) \citep{bellemare2017distributional}:
\vspace{-0.05in}
\[
Z^\pi(s,a) = \sum_{t=0}^{\infty} \gamma^t\,r(s_{t},a_{t}), 
\]
\vspace{-0.05in}
\[\mathcal{T}^\pi Z(s,a) \;\eqD\; r(s,a) + \gamma\,Z\bigl(s',\pi(s')\bigr),
\]
where \(V :\eqD U\) denotes equality of probability laws, indicating random variables \(\{V, U\}\) are distributed according to the same law.
A popular parameterization uses quantile regression: one approximates \(Z^\pi(s,a)\) by \(N\) quantiles \(\boldsymbol{\theta}
(s,a)=[\theta_1(s,a),...,\theta_N(s,a)]:\mathcal{S}\times \mathcal{A} \rightarrow \mathbb{R}^N\) at fractions (quantile levels) \(\tau_i = i/N\), for \(i=1,\dots,N\). In other words, the quantile distribution of \(Z^\pi(s,a)\) is represented by a uniform probability distribution supported on \( \{\theta_i(s,a)\}_{i=1}^N \): \(Z^\pi(s,a) = \frac{1}{N}\sum_{i=1}^N \delta_{\theta_i}(s,a)\) where \(\delta_{\theta_i}\) denotes a Dirac at \(\theta_i\). An example of quantile functions is illustrated in Fig.~\ref{fig:quantile-example}, with \(\theta\) and \(\tau\) indicated.

To update the critic, instead of formulating the TD error, one can minimize the quantile Huber loss \citep{dabney2018distributional} with threshold \(\kappa>0\):
\begin{equation}
\rho_\tau^\kappa(\delta)
=\bigl|\tau - \mathbf{1}\{\delta<0\}\bigr|\;H_\kappa(\delta), H_\kappa(\delta)
=\begin{cases}
\tfrac{1}{2}\,\delta^2, & |\delta|\le\kappa,\\
\kappa\,|\delta| \;-\;\tfrac{1}{2}\,\kappa^2, & |\delta|>\kappa.
\end{cases}
\end{equation}
In distributional RL with \(N\) quantile fractions \(\{\tau_i\}\), the loss for the critic is defined as   
\begin{equation} \label{eq:huber-loss}
\min_{\theta} \mathcal{L}_{\rm QR} (\theta)
= \min_{\theta} \frac{1}{N} \sum_{i=1}^N \sum_{j=1}^N \rho_{\tau_i} \left( \delta_{ij}\right), \delta_{ij}=r + \gamma\,\theta_j(s',a') - \theta_i(s,a).
\end{equation}
Once the return distribution is learned, one can optimize risk measures $M$, e.g.\ Conditional Value at Risk (CVaR) \citep{rockafellar2000optimization}, by maximizing
CVaR\(\bigl(Z^\pi\bigr)\) rather than \(\E[Z^\pi]\), yielding risk-sensitive policies.

\paragraph{Deterministic reward as a special case.}
If $q(\cdot \mid s,a)$ is a point mass at some value $r(s,a)$ for every $(s,a)$,
then we recover the usual deterministic reward setting. 
Thus, our framework strictly generalizes standard IRL.

\paragraph{Why distributions matter.}
If the reward is inherently stochastic (for example, due to noisy human judgments),
matching only the \emph{mean} reward or mean return is not enough to capture the full behavior.
Two policies can have the same expected return but very different risk profiles.
This motivates working with the full return distribution $Z^\pi$, not just its expectation.

\begin{table}[ht]
  \caption{Examples of distortion risk measures.}
  \label{table:spectral_risk}
  \centering
  \begin{tabular}{lll}
    \toprule
    Measure & $\xi(v)$ & Interpretation \\
    \midrule
    CVaR$_\alpha$ & $\min \left(v/\alpha, 1\right)$ & Average of worst $\alpha$-fraction of outcomes \\
    Wang's & $\Phi(\Phi^{-1}(v)+\lambda)$& \(\lambda>0\) risk-aversion, \(\lambda<0\) risk-seeking \\
    \bottomrule
  \end{tabular}
\end{table}

\subsection{First-Order Stochastic Dominance (FSD)}
\label{app:fsd-intuition}

We now recall first-order stochastic dominance, which provides a way to compare entire distributions,
not just means or variances.

\begin{definition}[First-order stochastic dominance]
Let $X$ and $Y$ be real-valued integrable random variables with cumulative distribution functions
$F_X$ and $F_Y$. We say that $X$ \emph{first-order stochastically dominates} $Y$, written
$X \succeq_{\mathrm{FSD}} Y$, if
\[
  F_X(z) \;\le\; F_Y(z) \quad \text{for all } z \in \mathbb{R}.
\]
\end{definition}

Intuitively, $X \succeq_{\mathrm{FSD}} Y$ means that $X$ tends to take larger values than $Y$:
for every threshold $z$, the probability that $X$ falls below $z$ is no larger than the probability
that $Y$ does. Graphically, the CDF of $X$ lies everywhere \emph{below} the CDF of $Y$.

\paragraph{Connection to utilities and mean dominance.}
A classical result states that $X \succeq_{\mathrm{FSD}} Y$ if and only if
\[
  \mathbb{E}[u(X)] \;\ge\; \mathbb{E}[u(Y)]
\]
for every non-decreasing utility function $u$. In particular, taking $u(x)=x$, we get
\[
  \mathbb{E}[X] \;\ge\; \mathbb{E}[Y],
\]
so FSD implies mean dominance. However, the converse is false: matching or exceeding the mean
does \emph{not} guarantee FSD.

\paragraph{FSD in our context.}
In our framework, we would like the return distribution of the expert policy, $Z^E$, to dominate
that of any learned policy $Z^\pi$, or vice versa depending on the formulation. 
This is a strong requirement and is typically hard to enforce directly during learning.
Our approach therefore designs an objective that \emph{penalizes violations of FSD}
and then turns this objective into an energy function for learning the reward distribution.

\subsection{The FSD Violation Objective as an Energy Function}
\label{app:fsd-energy}

Recall the FSD-based objective in the main text:
\begin{equation}
  \mathcal{L}_{\mathrm{FSD}}(\pi, r)
  \;=\;
  \int_{-\infty}^\infty \bigl[ F_{Z^E}(z) - F_{Z^\pi}(z) \bigr]_+ \, dz,
  \label{eq:app-fsd-loss}
\end{equation}
where $[x]_+ = \max\{x,0\}$ denotes the positive part. 
This quantity measures, in an integrated way, how much $F_{Z^E}$ lies \emph{above} $F_{Z^\pi}$.
If $Z^E \succeq_{\mathrm{FSD}} Z^\pi$, then $F_{Z^E}(z) \le F_{Z^\pi}(z)$ for all $z$, so the integrand
is always zero, and hence $\mathcal{L}_{\mathrm{FSD}}(\pi,r) = 0$.
If FSD is violated, then $\mathcal{L}_{\mathrm{FSD}}(\pi,r)$ becomes positive.

\paragraph{Energy-based interpretation.}
We treat $\mathcal{L}_{\mathrm{FSD}}(\pi,r)$ as an \emph{energy} that scores how well a reward function
$r$ explains the expert demonstrations under policy $\pi$. 
Lower $\mathcal{L}_{\mathrm{FSD}}$ means fewer FSD violations and thus better agreement with the expert.
This motivates defining an energy-based model (EBM)
\begin{equation}
  p(\mathcal{D} \mid r)
  \;\propto\;
  \exp\bigl( - \mathcal{L}_{\mathrm{FSD}}(\pi, r) \bigr),
  \label{eq:app-ebm-likelihood}
\end{equation}
where $\mathcal{D}$ denotes the expert data and the proportionality hides a (typically intractable)
normalizing constant. In words: reward functions that produce small FSD violations are exponentially
more likely under the expert data.

This construction gives us a \emph{likelihood} model for the reward $r$ given the data $\mathcal{D}$,
which we will combine with a prior over $r$ and then approximate via variational inference.

\subsection{Variational Inference and ELBO Derivation}
\label{app:elbo-derivation}

We now derive the variational objective used to learn the reward distribution.
We start from Bayes' rule:
\[
  p(r \mid \mathcal{D})
  \;=\;
  \frac{p(\mathcal{D} \mid r) \, p_0(r)}{p(\mathcal{D})},
\]
where $p_0(r)$ is a prior over reward functions and
\[
  p(\mathcal{D}) 
  = \int p(\mathcal{D} \mid r) \, p_0(r) \, dr
\]
is the evidence (marginal likelihood), which is typically intractable to compute or differentiate.

We introduce a variational family $q_\phi(r \mid s,a)$, parameterized by $\phi$, to approximate the true posterior
$p(r \mid \mathcal{D})$. To measure how close $q_\phi$ is to the true posterior, consider the KL divergence
\begin{align}
  \mathrm{KL}\bigl(q_\phi(r \mid s,a) \,\|\, p(r \mid \mathcal{D})\bigr)
  &= \mathbb{E}_{q_\phi} \Bigl[
       \log \frac{q_\phi(r \mid s,a)}{p(r \mid \mathcal{D})}
     \Bigr].
\end{align}
Plugging in Bayes' rule for $p(r \mid \mathcal{D})$ gives
\begin{align}
  \mathrm{KL}\bigl(q_\phi(r \mid s,a) \,\|\, p(r \mid \mathcal{D})\bigr)
  &= \mathbb{E}_{q_\phi} \Bigl[
       \log \frac{q_\phi(r \mid s,a)}{p(\mathcal{D} \mid r)\,p_0(r) / p(\mathcal{D})}
     \Bigr] \\
  &= \mathbb{E}_{q_\phi} \Bigl[
       \log q_\phi(r \mid s,a)
       - \log p(\mathcal{D} \mid r)
       - \log p_0(r)
       + \log p(\mathcal{D})
     \Bigr].
\end{align}
We can separate out the term that does not depend on $r$:
\begin{align}
  \mathrm{KL}\bigl(q_\phi(r \mid s,a) \,\|\, p(r \mid \mathcal{D})\bigr)
  &= \mathbb{E}_{q_\phi} \bigl[ \log q_\phi(r \mid s,a) - \log p(\mathcal{D} \mid r) - \log p_0(r) \bigr]
     \;+\; \log p(\mathcal{D}).
\end{align}
Rearranging terms yields
\begin{align}
  \log p(\mathcal{D})
  &= \mathbb{E}_{q_\phi} \bigl[ \log p(\mathcal{D} \mid r) + \log p_0(r) - \log q_\phi(r \mid s,a) \bigr]
     \;+\; \mathrm{KL}\bigl(q_\phi(r \mid s,a) \,\|\, p(r \mid \mathcal{D})\bigr).
\end{align}

Since KL is non-negative, we obtain the \emph{evidence lower bound} (ELBO):
\begin{align}
  \log p(\mathcal{D})
  &\ge
  \mathbb{E}_{q_\phi} \bigl[ \log p(\mathcal{D} \mid r) + \log p_0(r) - \log q_\phi(r \mid s,a) \bigr]
  \;=:\; \mathrm{ELBO}(\phi).
\end{align}
Equivalently,
\begin{equation}
  \mathrm{ELBO}(\phi)
  \;=\;
  \mathbb{E}_{q_\phi(r \mid s,a)}\bigl[\log p(\mathcal{D} \mid r)\bigr]
  - \mathrm{KL}\bigl(q_\phi(r \mid s,a) \,\|\, p_0(r)\bigr),
  \label{eq:app-elbo}
\end{equation}
which matches the expression in the main text.

\paragraph{From ELBO to our reward objective.}
Maximizing the ELBO is equivalent to minimizing its negative. 
Using the EBM likelihood from Eq.~\eqref{eq:app-ebm-likelihood},
\[
  \log p(\mathcal{D} \mid r)
  = -\mathcal{L}_{\mathrm{FSD}}(\pi,r) + \text{const},
\]
where the constant does not depend on $r$ and thus can be dropped for optimization. 
Substituting into Eq.~\eqref{eq:app-elbo} and ignoring constants, we obtain the objective
\begin{align}
  \min_\phi \mathcal{L}_r(\phi)
  &:= \min_\phi 
      \mathbb{E}_{q_\phi(r \mid s,a)}\bigl[ \mathcal{L}_{\mathrm{FSD}}(\pi,r) \bigr]
      + \mathrm{KL}\bigl(q_\phi(r \mid s,a) \,\|\, p_0(r)\bigr),
\end{align}
which is precisely Eq.~\ref{eq:reward-update} in the main text.
In other words, we learn the reward distribution by balancing two terms:
(i) the expected FSD violation under $q_\phi$, and
(ii) a regularization term that keeps $q_\phi$ close to the prior $p_0$.

\subsection{Quantiles and the FSD Loss}
\label{app:quantile-representation}

We now explain in more detail why the FSD loss in Eq.~\eqref{eq:app-fsd-loss} can be expressed
in terms of quantile functions, which leads to a practical way to estimate it via sampling.

\paragraph{Quantile function.}
For a random variable $X$ with CDF $F_X$, its (generalized) quantile function $F_X^{-1} : [0,1]\to\mathbb{R}$
is defined by
\[
  F_X^{-1}(v) = \inf\{ x \in \mathbb{R} \mid F_X(x) \ge v \}, \quad v \in (0,1).
\]
Intuitively, $F_X^{-1}(v)$ is the value such that a fraction $v$ of the mass of $X$ lies at or below it.

\paragraph{Key identity.}
We use the following identity (proved in Appendix~\ref{prop_A.1} of the main text):
\begin{equation}
  \int_{-\infty}^\infty [F_{Z^E}(z) - F_{Z^\pi}(z)]_+ \, dz
  \;=\;
  \int_0^1 \bigl[ F_{Z^\pi}^{-1}(v) - F_{Z^E}^{-1}(v) \bigr]_+ \, dv.
  \label{eq:app-quantile-identity}
\end{equation}
This shows that integrating the positive difference of the CDFs is equivalent to integrating
the positive difference of the \emph{quantiles}, but with the roles of expert and policy swapped 
inside the bracket.

\paragraph{Sketch of proof idea.}
The proof relies on two facts:
(i) an integral representation of the difference between two distributions in terms of their quantiles, and
(ii) a change of variables between $z$ and $v$ through the CDF/quantile mapping.
One can start from the left-hand side, partition the real line into regions where
$F_{Z^E}(z) \ge F_{Z^\pi}(z)$ and where the opposite holds, and then perform a change of variables
$z = F_{Z^\pi}^{-1}(v)$ (and similarly for the expert), carefully tracking the positive part.
We refer the reader to the detailed derivation in Appendix~\ref{prop_A.1}.

\paragraph{Monte Carlo approximation.}
The identity \eqref{eq:app-quantile-identity} is particularly useful because we can approximate
quantiles from samples. For example, to approximate $F_{Z^\pi}^{-1}$, we draw $N$ return samples
\[
  z_n = \sum_{t=0}^\infty \gamma^t r_t^{(n)}, 
  \qquad r_t^{(n)} \sim q_\phi(\cdot \mid s_t^{(n)}, a_t^{(n)}),
\]
and sort them to obtain order statistics
\[
  z_{(1)} \le z_{(2)} \le \cdots \le z_{(N)}.
\]
A simple empirical approximation of the quantile function is then
\[
  F_{Z^\pi}^{-1}\!\left(\frac{k}{N}\right) \approx z_{(k)}.
\]
In practice, we use such empirical quantiles (for both the expert and the learned policy)
to estimate the integral on the right-hand side of Eq.~\eqref{eq:app-quantile-identity} via a Riemann sum.

\subsection{Distortion Risk Measures and Their Relation to FSD}
\label{app:drm-intuition}

Finally, we explain how distortion risk measures (DRMs) provide a scalar, risk-sensitive summary
of a return distribution and how they relate to FSD.

\begin{definition}[Distortion function]
A distortion function is a non-decreasing function $\xi : [0,1] \to [0,1]$ such that
$\xi(0) = 0$ and $\xi(1) = 1$. Its \emph{dual distortion} is defined as
\[
  \tilde{\xi}(v) := 1 - \xi(1-v), \quad v \in [0,1].
\]
\end{definition}

\begin{definition}[Distortion risk measure]
For an integrable random variable $X$ and a distortion function $\xi$, the associated
distortion risk measure $M_\xi$ is defined by
\[
  M_\xi(X)
  \;=\;
  \int_0^1 F_X^{-1}(v) \, d\tilde{\xi}(v),
\]
where $F_X^{-1}$ is the quantile function of $X$.
\end{definition}

\paragraph{Intuition.}
The DRM $M_\xi(X)$ aggregates all quantiles of $X$ into a single scalar value, 
with weights determined by $d\tilde{\xi}(v)$. Different choices of $\xi$ emphasize different
parts of the distribution:
for example, a concave $\tilde{\xi}$ assigns more weight to \emph{lower} quantiles, 
which corresponds to risk-averse behavior.

\paragraph{Connection to FSD.}
It is known that if $X \succeq_{\mathrm{FSD}} Y$, then
\[
  M_\xi(X) \;\ge\; M_\xi(Y)
  \quad \text{for every distortion function } \xi.
\]
Furthermore, the converse holds if we require the inequality to hold for \emph{all} distortion functions:
if $M_\xi(X) \ge M_\xi(Y)$ for every distortion function $\xi$, then $X \succeq_{\mathrm{FSD}} Y$.
This shows that DRMs are tightly linked to FSD: they preserve the FSD ordering if we consider all possible distortions.

In our method, we exploit this relationship by replacing the intractable indicator-based weighting of quantiles
(from Eq.~\eqref{eq:iv} in the main text) with a tractable distortion-based weighting. 
This yields a risk-aware policy objective of the form
\[
  \max_{\varphi} \, M_\xi(Z^{\pi_\varphi}) + \mathcal{H}(\pi_\varphi),
\]
which can be optimized with standard policy gradient techniques while still encoding 
a meaningful notion of distributional dominance relative to the expert.

\paragraph{Approximation viewpoint.}
Optimizing $M_\xi(Z^{\pi_\varphi})$ for a \emph{single} distortion function $\xi$ does not guarantee FSD dominance by itself; it corresponds to a weaker condition.
However, as discussed in the main text, if one could optimize this objective for \emph{all} distortion functions simultaneously, then under mild assumptions the resulting policy would satisfy the original FSD-based objective. Our practical objective can therefore be viewed as an approximation that focuses on a particular, user-chosen notion of risk.

}
\section{Proofs}

\subsection{Proofs for sections 4}
We first wish to show that 
\begin{equation}
   \int_{-\infty}^\infty [F_{Z^E}(z) - F_{Z^\pi}(z)]_+ dz = \int_0^1 \bigl[F_{Z^\pi}^{-1}(v)- F_{Z^E}^{-1}(v)\bigr]_+ {d} v.
\end{equation}

\begin{proposition} \label{prop_A.1}
Let \( Z^\pi \) and \( Z^E \) be two real-valued integrable random variables with cumulative distribution functions \( F_{Z^\pi} \) and \( F_{Z^E} \), and corresponding quantile functions \( F_{Z^\pi}^{-1} \) and \( F_{Z^E}^{-1} \). Then we have
\[
\int_{-\infty}^\infty \left[F_{Z^E}(z) - F_{Z^\pi}(z)\right]_+ \, dz
= \int_0^1 \left[ F_{Z^\pi}^{-1}(v) - F_{Z^E}^{-1}(v) \right]_+ \, dv,
\]
where \( [x]_+ := \max(x, 0) \).
\end{proposition}
{
\begin{proof}
Note that
\begin{align*}
    \int_{-\infty}^\infty \left[ F_{Z^E}(z) - F_{Z^\pi}(z) \right]_+  dz &= \int_{-\infty}^\infty \int_0^1 \mathbbm{1}_{ F_{Z^E}(z) \geq v\geq F_{Z^\pi}(z) } dv  dz\\
    &= \int_0^1 \int_{-\infty}^\infty \mathbbm{1}_{ F_{Z^E}(z) \geq v\geq F_{Z^\pi}(z) } dv dz\\
    &= \int_0^1 \int_{-\infty}^\infty \mathbbm{1}_{ F_{Z^\pi}^{-1}(v) \geq z \geq F_{Z^E}^{-1}(v) } dv dz\\
    &=  \int_0^1 \left[ F_{Z^\pi}^{-1}(v) - F_{Z^E}^{-1}(v) \right]_+ dv
\end{align*}
The interchange of integrals are permitted by the Theorem of Fubini-Tonelli as everything is positive \citep{heil2019introduction}. Note that the definition of the quantile function \citep{gut2006probability} is:
\[
F^{-1}(v) := \inf_{z\in \mathbb{R}}\{ F(z)\geq v \}.
\]
\end{proof}
}
\sufficiencylemma*
\begin{proof} \label{proof:prop_4.6}
Define the difference in quantile functions:
\[
h(v) := F_X^{-1}(v) - F_Y^{-1}(v).
\]

Suppose for contradiction that the set
\[
A := \{ v \in [0,1] | h(v) < 0 \}
\]
has positive Borel measure, i.e., \( \mu(A) > 0 \).
Define a distortion function \( \tilde{\xi}_A \) whose derivative is:
\[
\tilde{\xi}_A'(v) = \begin{cases}
\frac{1}{\mu(A)} & \text{if } v \in A, \\
0 & \text{otherwise}.
\end{cases}
\]
Then \( \tilde{\xi}_A \) is a valid distortion function and satisfies \( \int_0^1 d\tilde{\xi}_A(v) = 1 \).
Note that
\[
\mathcal{M}_{{\xi}_A}(X) - \mathcal{M}_{{\xi}_A}(Y)
= \int_0^1 h(v) \, d\tilde{\xi}_A(v)
= \int_A h(v) \cdot \frac{1}{\mu(A)} \, dv < 0.
\]

This contradicts the assumption that \( \mathcal{M}_{\tilde{\xi}}(X) \geq \mathcal{M}_{\tilde{\xi}}(Y) \) for all distortion functions \( \tilde{\xi} \). Therefore, the set where \( F_X^{-1}(v) < F_Y^{-1}(v) \) must have measure zero.
Thus we have 
\[
F_X^{-1}(v) \geq F_Y^{-1}(v) \quad \text{for } v \in [0,1] \text{ almost everywhere (a.e.)}
\]
which implies
\[
F_X(z) \leq F_Y(z) \quad \text{for all } z \in \mathbb{R},
\]
since 
\begin{align*}
F_X(z) = P_X\left( X<z\right) &= \mu\left( \{v\in[0,1] | F_X^{-1}(v)\leq z\}\right)\\
&\leq \mu\left( \{v\in[0,1]\cap A^c | F_X^{-1}(v)\leq z\}\right) 
+ \mu\left( \{v\in[0,1]\cap A | F_X^{-1}(v)\leq z\}\right) \\
&= \mu\left( \{v\in[0,1]\cap A^c | F_X^{-1}(v)\leq z\}\right) \\
&\leq \mu\left( \{v\in[0,1]\cap A^c | F_Y^{-1}(v)\leq z\}\right)\\
&\leq \mu\left( \{v\in[0,1] | F_Y^{-1}(v)\leq z\}\right)\\
&= F_Y(z)
\end{align*}
The second inequality is due to the fact that for any \(z\),
\[
\{v\in[0,1]\cap A^c | F_X^{-1}(v)\leq z\} \subseteq \{v\in[0,1]\cap A^c | F_Y^{-1}(v)\leq z\}
\]
Hence,
\[
X \succeq_{\text{FSD}} Y.
\]
\end{proof}

{
\subsection{Convergence Analysis}
\label{sec:appendix_convergence}

This appendix provides complete derivations and proofs for the convergence results
summarized in Section~5. We work in the discounted MDP setting with finite action space
$\mathcal A$ and (possibly infinite) state space $\mathcal S$. All function norms are
$\|\cdot\|_\infty$ unless otherwise specified.

We first recall the risk--sensitive Bellman operator. For a fixed policy $\pi$,
reward parameter $\phi$, and bounded $Q:\mathcal S\times\mathcal A\to\mathbb R$,
we write
\begin{equation}
\label{eq:drm-bellman}
\begin{aligned}
(\mathcal T^{\pi}_{\xi,\phi} Q)(s,a)
&:= \E_\xi \big[ q_\phi(s,a) \big]
\;+\;
\gamma\,\E_{\xi,\,s'\sim P(\cdot| s,a),\,a'\sim \pi(\cdot| s')}
    \big[ Q(s',a') \big].
\end{aligned}
\end{equation}
Here the notation $\E_\xi[\cdot]$ denotes the one-step evaluation combining the
conditional expectation over the transition kernel and the dynamic distortion
risk measure $M_\xi$ (i.e.\ a nested, time-consistent dynamic risk mapping).
Under this formulation, $\mathcal T^{\pi}_{\xi,\phi}$ is precisely the DRM
Bellman operator: it preserves the Markov structure and is a $\gamma$-contraction
under mild axioms on $M_\xi$~\citep{ruszczynski2010risk}, guaranteeing a unique
fixed point $Q^{\xi}_{\phi,\pi}$ for each $(\phi,\pi)$.

\subsubsection{Assumptions}

We collect the standing assumptions used in the analysis.

\boundedapp*

This is standard in discounted RL and is enforced in our implementation by
clipping the reward range (via a scaled $\tanh$ nonlinearity). It ensures that
all risk-sensitive value functions are uniformly bounded.

\lawlipapp*

This assumption is mild for smooth neural parameterizations of $q_\phi(r| s,a)$
(e.g., skew-normal with smooth outputs for location, scale, and skew). It states that
small changes in the reward parameters $\phi$ cannot drastically change the reward
distribution, which is necessary for the critic and policy to track the moving reward model.

\drmapp*

For normalized distortion risk measures $M_\xi$ (including CVaR, Wang-type,
and more general spectral DRMs), these properties are standard and follow from
their integral representation in terms of quantile functions.

\subsubsection{Contraction of the nested DRM Bellman operator}

We now verify that $\mathcal T^{\pi}_{\xi,\phi}$ is a $\gamma$-contraction in the
sup norm. This is the risk-sensitive analogue of the standard Bellman contraction
and is a special instance of the general results on nested risk mappings
in~\citet{ruszczynski2010risk,kopa2023contractivity}.

\begin{lemma}[Contraction of $\mathcal T^{\pi}_{\xi,\phi}$]
\label{lem:contraction-app}
Under Assumptions~\ref{ass:bounded-app} and~\ref{ass:drm-app}, for any fixed
$(\phi,\pi)$ and any bounded $U,V:\mathcal S\times\mathcal A\to\mathbb R$,
\begin{equation}
\label{eq:bellman-contraction}
\big\|\mathcal T^{\pi}_{\xi, \phi} U - \mathcal T^{\pi}_{\xi, \phi} V\big\|_\infty
\;\le\; \gamma \,\|U - V\|_\infty.
\end{equation}
\end{lemma}

\begin{proof}
For any $(s,a)$, the immediate reward terms cancel, and we have
\begin{equation}
\label{eq:contraction-pointwise}
\begin{aligned}
&\big|(\mathcal T^{\pi}_{\xi,\phi} U)(s,a) - (\mathcal T^{\pi}_{\xi,\phi} V)(s,a)\big| \\
&\quad= \gamma \big|\E_{\xi,\,s'\sim P(\cdot| s,a)}[U(s',A') - V(s',A')]\big| \\
&\quad\le \gamma \,\E_{s'\sim P(\cdot| s,a)}
      \big[\,\big|M_\xi(U(s',A')-V(s',A')| s')\big|\,\big] \\
&\quad\le \gamma \,\E_{s'\sim P(\cdot| s,a)}
      \big[\|U-V\|_\infty\big]
      \;=\; \gamma \,\|U-V\|_\infty,
\end{aligned}
\end{equation}
where we used Assumption~\ref{ass:drm-app} (1-Lipschitzness) in the third line.
Taking the supremum over $(s,a)$ yields~\ref{eq:bellman-contraction}.
\end{proof}

By the Banach fixed-point theorem, we immediately obtain:

\begin{corollary}[Existence and uniqueness of the risk-sensitive critic]
\label{cor:unique-Q}
Under Assumptions~\ref{ass:bounded-app} and~\ref{ass:drm-app}, for each fixed
$(\phi,\pi)$ there exists a unique $Q^{\xi}_{\phi,\pi}$ solving
\begin{equation}
\label{eq:Q-fixed-point}
Q^{\xi}_{\phi,\pi} = \mathcal T^{\pi}_{\xi,\phi} Q^{\xi}_{\phi,\pi}.
\end{equation}
\end{corollary}

Moreover, the critic is uniformly bounded.

\begin{lemma}
\label{lem:bounded-Q}
Under Assumption~\ref{ass:bounded-app}, let $B_Q:=R_{\max}/(1-\gamma)$. Then
for all $(\phi,\pi)$,
\begin{equation}
\label{eq:Q-bound}
\big\|Q^{\xi}_{\phi,\pi}\big\|_\infty \;\le\; B_Q.
\end{equation}
\end{lemma}

\begin{proof}
By unfolding the fixed point~\ref{eq:Q-fixed-point} along trajectories and
using $|q_\phi(s,a)|\le R_{\max}$, we get for all $(s,a)$
\begin{equation}
\label{eq:Q-series-bound}
\big|Q^{\xi}_{\phi,\pi}(s,a)\big|
\;\le\; \sum_{t=0}^\infty \gamma^t R_{\max}
\;=\; \frac{R_{\max}}{1-\gamma}
\;=\; B_Q.
\end{equation}
Taking the supremum over $(s,a)$ yields~\ref{eq:Q-bound}.
\end{proof}

\subsubsection{Softmax Lipschitz properties}

We next relate $Q$-function errors to policy errors via the softmax parameterization.

\begin{lemma}
\label{lem:logsoftmax-lip}
Let $Q,Q':\mathcal A\to\mathbb R$ be two vectors of $Q$-values, and define
\begin{equation}
\label{eq:softmax-def}
\pi(a) = \frac{e^{Q(a)}}{\sum_b e^{Q(b)}},
\qquad
\pi'(a) = \frac{e^{Q'(a)}}{\sum_b e^{Q'(b)}}.
\end{equation}
Then
\begin{equation}
\label{eq:logsoftmax-lip}
\|\log\pi - \log\pi'\|_\infty \;\le\; 2\,\|Q-Q'\|_\infty.
\end{equation}
\end{lemma}

\begin{proof}
For any action $a$,
\begin{equation}
\label{eq:logpi-diff}
\begin{aligned}
\log\pi(a)
&= Q(a) - \log\sum_b e^{Q(b)}, \\
\log\pi'(a)
&= Q'(a) - \log\sum_b e^{Q'(b)}.
\end{aligned}
\end{equation}
Subtracting,
\begin{equation}
\label{eq:logpi-diff2}
\begin{aligned}
\log\pi(a)-\log\pi'(a)
&= \big(Q(a)-Q'(a)\big)
 - \Big(\log\sum_b e^{Q(b)} - \log\sum_b e^{Q'(b)}\Big).
\end{aligned}
\end{equation}
The log-sum-exp function is $1$-Lipschitz in $\|\cdot\|_\infty$, i.e.
\begin{equation}
\label{eq:lse-lip}
\Big|\log\sum_b e^{Q(b)} - \log\sum_b e^{Q'(b)}\Big|
\;\le\; \|Q-Q'\|_\infty.
\end{equation}
Combining~\ref{eq:logpi-diff2} and~\ref{eq:lse-lip} gives
\begin{equation}
\label{eq:logpi-pointwise}
|\log\pi(a)-\log\pi'(a)|
\;\le\; |Q(a)-Q'(a)| + \|Q-Q'\|_\infty
\;\le\; 2\,\|Q-Q'\|_\infty.
\end{equation}
Taking the supremum over $a$ yields~\ref{eq:logsoftmax-lip}.
\end{proof}

\subsubsection{Lipschitz sensitivity}

We now show that the DRM $Q$-function depends smoothly on the reward parameters
$\phi$, both for optimal control and for fixed-policy evaluation.

\begin{lemma}
\label{lem:Lq-opt}
Suppose Assumptions~\ref{ass:bounded-app}, \ref{ass:lawlip-app}, and
\ref{ass:drm-app} hold. 
Then for all $(s,a)$ and all $\phi_1,\phi_2$,
\begin{equation}
\label{eq:reward-mean-lip-bound}
\big|q_{\phi_1}(s,a)-q_{\phi_2}(s,a)\big|
\;\le\; L_R\,\|\phi_1-\phi_2\|.
\end{equation}
Let $M_\xi$ denote the nested distortion risk functional, and assume it is
$1$-Lipschitz in $\|\cdot\|_\infty$ as in Assumption~\ref{ass:drm-app}. Define
the optimal risk-sensitive $Q$-function for parameter $\phi$ by
\begin{equation}
\label{eq:Q-opt-def}
Q^{\xi}_{\phi,*}(s,a)
:= \sup_{\pi} M_\xi\!\Big(
    \sum_{t=0}^{\infty} \gamma^t r_\phi(s_t,a_t)
    \,\Big|\, s_0=s,a_0=a,\pi
\Big),
\end{equation}
where $\{(s_t,a_t)\}_{t\ge 0}$ is the trajectory under policy $\pi$ starting
from $(s_0,a_0)=(s,a)$. Then there exists
\begin{equation}
\label{eq:Lq-def}
L_q := \frac{L_R}{1-\gamma}
\end{equation}
such that for all $\phi_1,\phi_2$,
\begin{equation}
\label{eq:Lq-opt-bound}
\big\|Q^{\xi}_{\phi_1,*} - Q^{\xi}_{\phi_2,*}\big\|_\infty
\;\le\; L_q \,\|\phi_1 - \phi_2\|.
\end{equation}
\end{lemma}

\begin{proof}
The bound on the reward smoothness is immediately due to assumption \ref{ass:lawlip-app}.
Fix $\phi_1,\phi_2$ and $(s,a)$. For any policy $\pi$, let $\{(s_t,a_t)\}_{t\ge 0}$
be the trajectory under $\pi$ with $(s_0,a_0)=(s,a)$, and define
\begin{equation}
\label{eq:Gphi-def}
G_{\phi_i}^{\pi}
:= \sum_{t=0}^{\infty} \gamma^t q_{\phi_i}(s_t,a_t),
\qquad i\in\{1,2\}.
\end{equation}
By definition~\ref{eq:Q-opt-def},
\begin{equation}
\label{eq:Qopt-as-sup}
Q^{\xi}_{\phi_i,*}(s,a)
= \sup_{\pi} M_\xi\big(G_{\phi_i}^{\pi} \,\big|\, s,a,\pi\big),
\quad i\in\{1,2\}.
\end{equation}
Using the inequality
\begin{equation}
\label{eq:sup-sup-ineq}
\big|\sup_\pi f_\pi - \sup_\pi g_\pi\big|
\;\le\; \sup_\pi |f_\pi-g_\pi|,
\end{equation}
we obtain
\begin{equation}
\label{eq:Qopt-diff1}
\begin{aligned}
\big|Q^{\xi}_{\phi_1,*}(s,a) - Q^{\xi}_{\phi_2,*}(s,a)\big|
&= \Big|\sup_{\pi} M_\xi(G_{\phi_1}^{\pi}| s,a,\pi)
     - \sup_{\pi} M_\xi(G_{\phi_2}^{\pi}| s,a,\pi)\Big| \\
&\le \sup_{\pi} \big|M_\xi(G_{\phi_1}^{\pi}| s,a,\pi)
                     - M_\xi(G_{\phi_2}^{\pi}| s,a,\pi)\big|.
\end{aligned}
\end{equation}
For each fixed $\pi$, the 1-Lipschitz property of $M_\xi$ in $\|\cdot\|_\infty$
(Assumption~\ref{ass:drm-app}) gives
\begin{equation}
\label{eq:Mxi-lip-trajectory}
\begin{aligned}
\big|M_\xi(G_{\phi_1}^{\pi}| s,a,\pi)
   - M_\xi(G_{\phi_2}^{\pi}| s,a,\pi)\big|
&\le \big\|G_{\phi_1}^{\pi} - G_{\phi_2}^{\pi}\big\|_\infty \\
&= \sup_{\omega}
   \Bigg|\sum_{t=0}^{\infty} \gamma^t
       \big(q_{\phi_1}(s_t(\omega),a_t(\omega))
         - q_{\phi_2}(s_t(\omega),a_t(\omega))\big)\Bigg| \\
&\le \sum_{t=0}^{\infty} \gamma^t
       \sup_{(s',a')} \big|q_{\phi_1}(s',a') - q_{\phi_2}(s',a')\big| \\
&\le \sum_{t=0}^{\infty} \gamma^t\,L_R \|\phi_1-\phi_2\| \\
&= \frac{L_R}{1-\gamma}\,\|\phi_1-\phi_2\|.
\end{aligned}
\end{equation}
The bound does not depend on $\pi$, so combining it with Eq.~\ref{eq:Qopt-diff1} we obtain
\begin{equation}
\label{eq:Qopt-diff2}
\big|Q^{\xi}_{\phi_1,*}(s,a) - Q^{\xi}_{\phi_2,*}(s,a)\big|
\;\le\; \frac{L_R}{1-\gamma}\,\|\phi_1-\phi_2\|.
\end{equation}
Taking the supremum over $(s,a)$ yields the desired result.
\end{proof}

\begin{lemma}[Lipschitz continuity of $Q^{\xi}_{\phi,\pi}$ in $\phi$ for fixed policy]
\label{lem:Lq-fixedpi}
Suppose Assumptions~\ref{ass:bounded-app}, \ref{ass:lawlip-app}, and
\ref{ass:drm-app} hold, and fix any stationary policy $\pi$. Define the
risk--sensitive evaluation $Q$-function as
\begin{equation}
\label{eq:Qeval-def}
Q^{\xi}_{\phi,\pi}(s,a)
:= M_\xi\!\Big(
    \sum_{t=0}^{\infty} \gamma^t q_\phi(s_t,a_t)
    \,\Big|\, s_0=s,a_0=a,\pi
\Big),
\end{equation}
where $\{(s_t,a_t)\}_{t\ge 0}$ is the trajectory under $\pi$ starting from
$(s_0,a_0)=(s,a)$. Then for all $\phi_1,\phi_2$,
\begin{equation}
\label{eq:Lq-fixedpi-bound}
\big\|Q^{\xi}_{\phi_1,\pi} - Q^{\xi}_{\phi_2,\pi}\big\|_\infty
\;\le\; L_q \,\|\phi_1 - \phi_2\|,
\qquad
L_q := \frac{L_R}{1-\gamma}.
\end{equation}
\end{lemma}

\begin{proof}
Fix $\pi$ and $(s_0,a_0)=(s,a)$, and let $\{(s_t,a_t)\}_{t\ge 0}$ be the
trajectory under $\pi$. For $i\in\{1,2\}$, define $G_{\phi_i}^{\pi}$ as in
\ref{eq:Gphi-def}. Then by~\ref{eq:Qeval-def},
\begin{equation}
\label{eq:Qeval-Mxi}
Q^{\xi}_{\phi_i,\pi}(s,a)
= M_\xi(G_{\phi_i}^{\pi}| s,a,\pi),
\qquad i\in\{1,2\}.
\end{equation}
Thus
\begin{equation}
\label{eq:Qeval-diff1}
\begin{aligned}
\big|Q^{\xi}_{\phi_1,\pi}(s,a) - Q^{\xi}_{\phi_2,\pi}(s,a)\big|
&= \big|M_\xi(G_{\phi_1}^{\pi}| s,a,\pi)
        - M_\xi(G_{\phi_2}^{\pi}| s,a,\pi)\big| \\
&\le \big\|G_{\phi_1}^{\pi} - G_{\phi_2}^{\pi}\big\|_\infty \\
&\le \frac{L_R}{1-\gamma}\,\|\phi_1-\phi_2\|,
\end{aligned}
\end{equation}
where the last inequality is identical to the bound in
\ref{eq:Mxi-lip-trajectory}. Taking the supremum over $(s,a)$ gives
\ref{eq:Lq-fixedpi-bound}.
\end{proof}

\subsubsection{One-step critic recursion}

We now derive a simple one-step recursion for the critic's tracking error
as the reward parameters $\phi_k$ and policies $\pi_k$ evolve across iterations.

For each iteration $k$, define
\begin{equation}
\label{eq:Ek-def}
E_k
:= \big\|Q^{\xi}_{\phi_k,\pi_k} - Q^{\xi}_{\phi_k,\pi^\star_{\phi_k}}\big\|_\infty,
\end{equation}
where $\pi^\star_{\phi_k}$ is an optimal DRM policy for reward parameter
$\phi_k$, i.e.
\begin{equation}
\label{eq:pi-star-def}
\pi^\star_{\phi_k} \propto \text{softmax}_{\pi} Q^{\xi}_{\phi_k,\pi}.
\end{equation}

\begin{lemma}
\label{lem:onestep-app}
Suppose Assumptions~\ref{ass:bounded-app}, \ref{ass:lawlip-app}, and
\ref{ass:drm-app} hold, and let $L_q$ be as in Lemma~\ref{lem:Lq-fixedpi}.
Then for all $k\ge 1$,
\begin{equation}
\label{eq:onestep-recursion}
E_k
\;\le\;
\gamma\,E_{k-1}
\;+\;
2 L_q \,\|\phi_k - \phi_{k-1}\|.
\end{equation}
\end{lemma}

\begin{proof}
Add and subtract $Q^{\xi}_{\phi_{k-1},\pi_k}$ and
$Q^{\xi}_{\phi_{k-1},\pi^\star_{\phi_{k-1}}}$ inside the norm:
\begin{align}
\label{eq:onestep-split}
\begin{aligned}
&\big\| Q^{\xi}_{\phi_k,\pi_k} - Q^{\xi}_{\phi_k,\pi^\star_{\phi_k}} \big\|_\infty \\
&\quad= \big\| Q^{\xi}_{\phi_k,\pi_k} - Q^{\xi}_{\phi_k,\pi^\star_{\phi_k}}
       + Q^{\xi}_{\phi_{k-1},\pi_k} - Q^{\xi}_{\phi_{k-1},\pi_k}
       + Q^{\xi}_{\phi_{k-1},\pi^\star_{\phi_{k-1}}}
         - Q^{\xi}_{\phi_{k-1},\pi^\star_{\phi_{k-1}}} \big\|_\infty \\
&\quad\le
    \big\|Q^{\xi}_{\phi_{k-1},\pi^\star_{\phi_{k-1}}}
           - Q^{\xi}_{\phi_k,\pi^\star_{\phi_k}} \big\|_\infty
  + \big\|Q^{\xi}_{\phi_k,\pi_k}
           - Q^{\xi}_{\phi_{k-1},\pi_k} \big\|_\infty
  + \big\| Q^{\xi}_{\phi_{k-1},\pi_k}
           - Q^{\xi}_{\phi_{k-1},\pi^\star_{\phi_{k-1}}} \big\|_\infty.
\end{aligned}
\end{align}
By Lemma~\ref{lem:Lq-opt} (with $\pi^\star_{\phi_{k-1}}$ and $\pi^\star_{\phi_k}$
both optimal) and Lemma~\ref{lem:Lq-fixedpi} (with $\pi=\pi_k$), we have
\begin{equation}
\label{eq:onestep-Lq-bounds}
\begin{aligned}
\big\|Q^{\xi}_{\phi_{k-1},\pi^\star_{\phi_{k-1}}}
       - Q^{\xi}_{\phi_k,\pi^\star_{\phi_k}} \big\|_\infty
&\le L_q \,\|\phi_k-\phi_{k-1}\|, \\
\big\|Q^{\xi}_{\phi_k,\pi_k}
       - Q^{\xi}_{\phi_{k-1},\pi_k} \big\|_\infty
&\le L_q \,\|\phi_k-\phi_{k-1}\|.
\end{aligned}
\end{equation}
Therefore,
\begin{equation}
\label{eq:onestep-bound-phi}
\big\| Q^{\xi}_{\phi_k,\pi_k} - Q^{\xi}_{\phi_k,\pi^\star_{\phi_k}} \big\|_\infty
\;\le\;
2 L_q \,\|\phi_k - \phi_{k-1}\|
\;+\;
\big\| Q^{\xi}_{\phi_{k-1},\pi_k}
       - Q^{\xi}_{\phi_{k-1},\pi^\star_{\phi_{k-1}}} \big\|_\infty.
\end{equation}

Next observe that for fixed $\phi_{k-1}$, $\pi^\star_{\phi_{k-1}}$ is optimal, so
\begin{equation}
\label{eq:Q-ordering}
Q^{\xi}_{\phi_{k-1},\pi_k}
\;\le\; Q^{\xi}_{\phi_{k-1},\pi^\star_{\phi_{k-1}}}
\quad\text{pointwise}.
\end{equation}
Moreover, by monotonicity of the Bellman operator and
Lemma~\ref{lem:contraction-app},
\begin{equation}
\label{eq:Bellman-ordering}
0 \;\le\;
Q^{\xi}_{\phi_{k-1},\pi^\star_{\phi_{k-1}}}
 - Q^{\xi}_{\phi_{k-1},\pi_k}
\;\le\;
\mathcal T^{\pi_k}_{\xi,\phi_{k-1}}
  \big(Q^{\xi}_{\phi_{k-1},\pi^\star_{\phi_{k-1}}}
      - Q^{\xi}_{\phi_{k-1},\pi_k}\big),
\end{equation}
so taking norms and using~\ref{eq:bellman-contraction} gives
\begin{equation}
\label{eq:Bellman-contracted-gap}
\big\|Q^{\xi}_{\phi_{k-1},\pi^\star_{\phi_{k-1}}}
     - Q^{\xi}_{\phi_{k-1},\pi_k}\big\|_\infty
\;\le\;
\gamma \big\|Q^{\xi}_{\phi_{k-1},\pi^\star_{\phi_{k-1}}}
           - Q^{\xi}_{\phi_{k-1},\pi_{k-1}}\big\|_\infty
\;=\; \gamma\,E_{k-1}.
\end{equation}
So that we get
\begin{equation}
\label{eq:onestep-final}
E_k
= \big\| Q^{\xi}_{\phi_k,\pi_k} - Q^{\xi}_{\phi_k,\pi^\star_{\phi_k}} \big\|_\infty
\;\le\;
\gamma\,E_{k-1}
\;+\;
2 L_q \,\|\phi_k - \phi_{k-1}\|,
\end{equation}
as claimed.
\end{proof}

\subsubsection{Smooth reward updates and averaged critic tracking}

We now relate the parameter drift $\|\phi_k-\phi_{k-1}\|$ to the reward update
objective $\mathcal L_r(\phi)$ used in Eq.~\ref{eq:reward-update}.

\begin{assumption}[Smoothness and bounded gradients of the reward objective]
\label{ass:smooth-grad}
Let $\mathcal L_r(\phi)$ denote the reward-distribution objective in
Eq.~\ref{eq:reward-update}. Assume:
\begin{enumerate}
    \item $\mathcal L_r$ is differentiable and its gradient is $L_\nabla$--Lipschitz:
    \begin{equation}
    \label{eq:grad-lip}
    \big\|\nabla \mathcal L_r(\phi_1) - \nabla \mathcal L_r(\phi_2)\big\|
    \;\le\; L_\nabla \,\|\phi_1-\phi_2\|
    \quad\text{for all }\phi_1,\phi_2.
    \end{equation}
    \item The iterates $\{\phi_k\}$ are projected onto a compact set
    $\Phi \subset \mathbb R^d$, so that
    \begin{equation}
    \label{eq:Gmax-def}
    G_{\max}
    \;:=\; \sup_{\phi\in\Phi} \big\|\nabla \mathcal L_r(\phi)\big\|
    \;<\; \infty.
    \end{equation}
\end{enumerate}
The reward update step is
\begin{equation}
\label{eq:phi-update}
\phi_k
= \Pi_\Phi\!\big(\phi_{k-1} - \eta_{k-1}\,\nabla \mathcal L_r(\phi_{k-1})\big),
\end{equation}
where $\Pi_\Phi$ is the Euclidean projection onto $\Phi$ and
$\{\eta_k\}$ is a deterministic stepsize schedule. We state the assumption with projection for generality; in our implementation, the parameterization and reward clipping make this projection implicit.
\end{assumption}

\begin{lemma}
\label{lem:phi-drift}
Under Assumption~\ref{ass:smooth-grad},
\begin{equation}
\label{eq:phi-drift}
\|\phi_k - \phi_{k-1}\|
\;\le\; \eta_{k-1}\,G_{\max}.
\end{equation}
\end{lemma}

\begin{proof}
By non-expansiveness of the projection,
\begin{equation}
\label{eq:proj-nonexp}
\begin{aligned}
\|\phi_k - \phi_{k-1}\|
&= \big\|\Pi_\Phi(\phi_{k-1} - \eta_{k-1}\nabla \mathcal L_r(\phi_{k-1}))
      - \Pi_\Phi(\phi_{k-1})\big\| \\
&\le \eta_{k-1}\,\big\|\nabla \mathcal L_r(\phi_{k-1})\big\|
\;\le\; \eta_{k-1}\,G_{\max},
\end{aligned}
\end{equation}
which is~\ref{eq:phi-drift}.
\end{proof}

Now we are ready to get the main recursion formula.

\qconvnew*

\begin{proof}
By Lemmas~\ref{lem:onestep-app} and~\ref{lem:phi-drift},
\begin{equation}
\label{eq:Ek-alpha}
E_k \;\le\; \gamma E_{k-1} + 2 L_q G_{\max}\eta_{k-1}.
\end{equation}
Taking the sum, we have
\begin{equation}
\label{eq:Ek-unroll}
\sum_{k=1}^K E_k
\;\le\; \sum_{k=1}^K \gamma  E_{k-1}
+ 2 L_q G_{\max} \sum_{k=1}^K  \eta_{k-1}.
\end{equation}
Rearrange and average over $K$ gives
\begin{equation}
\label{eq:avg-Ek-step1}
\frac{1-\gamma}{K}\sum_{k=1}^K E_k
\;\le\;
\frac{\gamma}{ K} \left(E_0 - E_K \right)
+ 2 L_q G_{\max} \eta K^{-\sigma}.
\end{equation}
Dividing both sides by $1-\gamma$ gives
\begin{align}
    \frac{1}{K}\sum_{k=1}^K E_k \leq \frac{\gamma}{(1-\gamma) K} C_0
+ \frac{1}{1-\gamma}2 L_q G_{\max} \eta K^{-\sigma} .
\end{align}
This proves the claim.
\end{proof}

\subsubsection{Policy convergence in log-probability}

Finally, we transfer the critic tracking guarantees to the induced policies.

\policyconvnew*
\begin{proof}
Fix $k$ and $s$. Let
\begin{equation}
\label{eq:x-y-def}
x(\cdot) = Q^{\xi}_{\phi_k,\pi_k}(s,\cdot),
\qquad
y(\cdot) = Q^{\xi}_{\phi_k,\pi^\star_{\phi_k}}(s,\cdot).
\end{equation}
By Lemma~\ref{lem:logsoftmax-lip},
\begin{equation}
\label{eq:policy-gap-state}
\big\|\log \pi_k^{+}(\cdot| s) - \log \pi^{\star+}_{\phi_k}(\cdot| s)\big\|_\infty
\;\le\; 2 \|x-y\|_\infty.
\end{equation}
Taking the supremum over $s$ yields
\begin{equation}
\label{eq:policy-gap-global}
\big\|\log \pi_k^{+} - \log \pi^{\star+}_{\phi_k}\big\|_\infty
\;\le\;
2 \big\|Q^{\xi}_{\phi_k,\pi_k} - Q^{\xi}_{\phi_k,\pi^\star_{\phi_k}}\big\|_\infty
\;=\; 2 E_k.
\end{equation} 
Averaging over $k=1,\dots,K$ and
substituting the bound from Theorem~\ref{thm:q-conv-new} gives
Eq.~\ref{eq:policy-gap-avg}.
\end{proof}

\subsubsection{First--order convergence of the reward update}

We now show that, under mild additional conditions, the reward update drives the gradient of the reward objective to zero in an averaged sense, so that the iterates approach a stationary point of the inner minimization problem over $\phi$. This is the strongest guarantee available without assuming that the reward function approximator is convex. 

Recall that the reward objective $\mathcal{L}_r(\phi)$ and its update rule
were introduced in Assumption~\ref{ass:smooth-grad}.  The update at
iteration $k$ is
\begin{equation}
    \phi_{k+1}
    = \Pi_\Phi\!\bigl(\phi_k - \eta_k g_k\bigr),
\end{equation}
where $g_k$ is the stochastic gradient computed using the current critic
$Q^{\xi}_{\phi_k,\pi_k}$ and policy $\pi_k$.

\begin{assumption}[Gradient estimator and critic bias]
\label{ass:grad-bias}
Let $\mathcal{F}_k$ denote the filtration generated by all randomness up
to iteration $k$.
Assume that the stochastic gradient $g_k$ satisfies, for some constants
$C_g,G_g>0$,
\begin{align}
    \bigl\|\E[g_k | \mathcal{F}_k] - \nabla\mathcal{L}_r(\phi_k)\bigr\|
    &\le C_g\,E_k, \label{eq:grad-bias-bound} \\
    \E\bigl[\|g_k\|^2\bigr] &\le G_g^2, \label{eq:grad-second-moment}
\end{align}
where
\begin{equation}
    E_k
    := \bigl\|Q^{\xi}_{\phi_k,\pi_k} - Q^{\xi}_{\phi_k,\pi^\star_{\phi_k}}\bigr\|_\infty
\end{equation}
is the critic tracking error defined above.
\end{assumption}

Intuitively, \eqref{eq:grad-bias-bound} states that the gradient bias
vanishes as soon as the critic tracks the DRM--optimal $Q$ well (i.e.,
$E_k$ is small), which is consistent with the inequality in
\eqref{eq:policy-gap-avg}: a small critic gap implies a small
occupancy--measure mismatch, hence a small gradient bias.  The
second--moment bound \eqref{eq:grad-second-moment} is standard in
nonconvex stochastic optimization.

\firstorder*

\begin{proof}
We begin from the smoothness inequality with
$\phi'=\phi_{k+1}$, $\phi=\phi_k$:
\begin{align}
    \mathcal L_r(\phi_{k+1})
    &\le
    \mathcal L_r(\phi_k)
    + \big\langle \nabla\mathcal L_r(\phi_k),
                 \phi_{k+1}-\phi_k\big\rangle
    + \frac{L_\nabla}{2}\,\|\phi_{k+1}-\phi_k\|^2.
    \label{eq:station-1}
\end{align}
By the non-expansiveness of the projection $\Pi_\Phi$ and the update rule,
\begin{align}
    \|\phi_{k+1}-\phi_k\|
    &= \big\|\Pi_\Phi(\phi_k - \eta_k g_k) - \Pi_\Phi(\phi_k)\big\|
    \nonumber\\
    &\le \eta_k \|g_k\|.
    \label{eq:projection-1}
\end{align}
Moreover,
\begin{align}
    \big\langle \nabla\mathcal L_r(\phi_k),\phi_{k+1}-\phi_k\big\rangle
    &= \big\langle \nabla\mathcal L_r(\phi_k),
                  \Pi_\Phi(\phi_k - \eta_k g_k)-\phi_k\big\rangle
    \nonumber\\
    &\le \big\langle \nabla\mathcal L_r(\phi_k),
                  -\eta_k g_k\big\rangle \\
    &= -\eta_k \big\langle \nabla\mathcal L_r(\phi_k), g_k\big\rangle.
    \label{eq:inner-prod-1}
\end{align}
Substituting the above into Eq.~\ref{eq:station-1} yields
\begin{align}
    \mathcal L_r(\phi_{k+1})
    &\le
    \mathcal L_r(\phi_k)
    - \eta_k \big\langle \nabla\mathcal L_r(\phi_k), g_k\big\rangle
    + \frac{L_\nabla}{2} \eta_k^2 \|g_k\|^2.
    \label{eq:station-2}
\end{align}
Taking conditional expectation and expanding the inner product gives
\[
\big\langle \nabla\mathcal L_r(\phi_k), \E[g_k|\mathcal F_k]\big\rangle
= \frac{1}{2}\Big(
    \|\nabla\mathcal L_r(\phi_k)\|^2
    + \|\E[g_k|\mathcal F_k]\|^2
    - \|\E[g_k|\mathcal F_k] - \nabla\mathcal L_r(\phi_k)\|^2
\Big),
\]
which gives
\begin{align}
    - \eta_k \big\langle \nabla\mathcal L_r(\phi_k), \E[g_k|\mathcal F_k]\big\rangle
    &= -\frac{\eta_k}{2}\|\nabla\mathcal L_r(\phi_k)\|^2
       -\frac{\eta_k}{2}\|\E[g_k|\mathcal F_k]\|^2
       +\frac{\eta_k}{2}\|\E[g_k|\mathcal F_k] - \nabla\mathcal L_r(\phi_k)\|^2.
    \label{eq:inner-expansion}
\end{align}
Substituting Eq.~\ref{eq:inner-expansion} into Eq.~\ref{eq:station-2} we obtain
\begin{align}
    \mathcal L_r(\phi_{k+1})
    &\le
    \mathcal L_r(\phi_k)
    -\frac{\eta_k}{2}\|\nabla\mathcal L_r(\phi_k)\|^2
    -\frac{\eta_k}{2}\|\E[g_k|\mathcal F_k]\|^2
    +\frac{\eta_k}{2}\|\E[g_k|\mathcal F_k] - \nabla\mathcal L_r(\phi_k)\|^2
    + \frac{L_\nabla}{2}\eta_k^2\E[\|g_k\|^2|\mathcal F_k]
    \nonumber\\
    &\le
    \mathcal L_r(\phi_k)
    -\frac{\eta_k}{2}\|\nabla\mathcal L_r(\phi_k)\|^2
    +\frac{\eta_k}{2}\|\E[g_k|\mathcal F_k] - \nabla\mathcal L_r(\phi_k)\|^2
    + \frac{L_\nabla}{2}\eta_k^2\E[\|g_k\|^2|\mathcal F_k]
    \label{eq:station-3}
\end{align}
where we discarded the negative term $-\frac{\eta_k}{2}\|\E[g_k|\mathcal F_k]\|^2$.
Next we bound the bias term.
Condition on $\phi_k$, and use $\|\E[g_k|\mathcal F_k] - \nabla\mathcal L_r(\phi_k)\|
\le C_g E_k$:
\begin{align}
    \E\big[\|\E[g_k|\mathcal F_k] - \nabla\mathcal L_r(\phi_k)\|^2\big]
    &\le \E\big[C_g^2 E_k^2\big]
    \nonumber\\
    &\le C_g^2\,\E[E_k^2]
    \;\le\; C C_g^2\,\E[E_k],
    \label{eq:bias-bound}
\end{align}
where we used $E_k\ge 0$ and $E_k\le C'\|Q\|_\infty$ so that $E_k^2\le C E_k$.
Similarly,
\begin{equation}
    \label{eq:gk-norm-bound}
    \E\big[\|g_k\|^2\big] \le G_g^2.
\end{equation}
Taking expectations of Eq.~\ref{eq:station-3} and applying Eq.~\ref{eq:bias-bound}-\ref{eq:gk-norm-bound} gives
\begin{align}
    \E[\mathcal L_r(\phi_{k+1})]
    &\le
    \E[\mathcal L_r(\phi_k)]
    -\frac{\eta_k}{2}
       \E\big[\|\nabla\mathcal L_r(\phi_k)\|^2\big]
    +\frac{\eta_k}{2} C_g^2\,\E[E_k]
    +\frac{L_\nabla}{2}\eta_k^2 G_g^2.
    \label{eq:station-4}
\end{align}
Rearrange Eq.~\ref{eq:station-4} as
\begin{align}
    \frac{\eta_k}{2}
       \E\big[\|\nabla\mathcal L_r(\phi_k)\|^2\big]
    &\le
    \E[\mathcal L_r(\phi_k)] - \E[\mathcal L_r(\phi_{k+1})]
    +\frac{\eta_k}{2} C C_g^2\,\E[E_k]
    +\frac{L_\nabla}{2}\eta_k^2 G_{\max}^2.
    \label{eq:station-5}
\end{align}
Multiply both sides by $2/\eta_k$:
\begin{align}
    \E\big[\|\nabla\mathcal L_r(\phi_k)\|^2\big]
    &\le
    \frac{2}{\eta_k}
    \big(\E[\mathcal L_r(\phi_k)] - \E[\mathcal L_r(\phi_{k+1})]\big)
    + C_g^2\,\E[E_k]
    + L_\nabla \eta_k G_g^2.
    \label{eq:station-6}
\end{align}

Now sum Eq.~\ref{eq:station-6} over $k=0,\dots,K-1$:
\begin{align}
    \sum_{k=0}^{K-1}
    \E\big[\|\nabla\mathcal L_r(\phi_k)\|^2\big]
    &\le
    2\sum_{k=0}^{K-1}
        \frac{\E[\mathcal L_r(\phi_k)] - \E[\mathcal L_r(\phi_{k+1})]}{\eta_k}
    + C_g^2 \sum_{k=0}^{K-1}\E[E_k]
    + L_\nabla G_g^2 \sum_{k=0}^{K-1} \eta_k.
    \label{eq:station-7}
\end{align}
Using the boundedness equation and the fact that $\eta_k$ is constant, we bound the first sum as
\begin{align}
    \sum_{k=0}^{K-1}
        \frac{\E[\mathcal L_r(\phi_k)] - \E[\mathcal L_r(\phi_{k+1})]}{\eta_k}
    &= \sum_{k=0}^{K-1}
        \Big(\E[\mathcal L_r(\phi_k)] - \E[\mathcal L_r(\phi_{k+1})]\Big)
           \frac{1}{\eta_k}
    \nonumber\\
    &=
    \frac{K^{\sigma}}{\eta}
    \sum_{k=0}^{K-1}
        \Big(\E[\mathcal L_r(\phi_k)] - \E[\mathcal L_r(\phi_{k+1})]\Big)
    \nonumber\\
    &= \frac{K^{\sigma}}{\eta}
       \Big(\E[\mathcal L_r(\phi_0)] - \E[\mathcal L_r(\phi_K)]\Big)
    \nonumber\\
    &\le \frac{\mathcal L_{\max}-\mathcal L_{\min}}{\eta} K^{\sigma}.
    \label{eq:station-8}
\end{align}
Next, apply the averaged tracking bound on the action-value function:
\begin{align}
    \sum_{k=0}^{K-1}\E[E_k]
    &= K \cdot \frac{1}{K}\sum_{k=0}^{K-1}\E[E_k]
    \nonumber\\
    &= \mathcal O\big(1\big) + \mathcal{O}(K^{1-\sigma})
    \label{eq:station-9}
\end{align}
Finally, since $\eta_k = \eta K^{-\sigma}$ with $0<\sigma<1$,
\begin{align}
    \sum_{k=0}^{K-1}\eta_k
    &= \eta \sum_{k=1}^{K} K^{-\sigma}
    =   \mathcal O\big(K^{1-\sigma}\big).
    \label{eq:station-10}
\end{align}

Divide both sides by $K$:
\begin{align}
    \frac{1}{K}\sum_{k=0}^{K-1}
    \E\big[\|\nabla\mathcal L_r(\phi_k)\|^2\big]
    &\le
    \frac{2(\mathcal L_{\max}-\mathcal L_{\min})}{\eta} K^{\sigma-1} \\
    & + \frac{\gamma}{(1-\gamma)} C_0 C_g^2 / K
+ \frac{C_g^2}{1-\gamma}2 L_q G_{\max} \eta K^{-\sigma}\\
& + L_\nabla G_g^2\,\eta K^{- 1}.
    \label{eq:station-12}
\end{align}
So that we have 
\begin{equation}
    \frac{1}{K}\sum_{k=0}^{K-1}
    \E\big[\|\nabla\mathcal L_r(\phi_k)\|^2\big]
    \;=\;
    \mathcal{O}(K^{-\sigma}) +  \mathcal{O}(K^{-1+\sigma}) +  \mathcal{O}(K^{-1}),
\end{equation}
Since $0<\sigma<1$, all three terms vanish as $K\to\infty$, so the
averaged squared gradient converges to zero.
\end{proof}

}
\section{Model architecture and hyper-parameters} \label{model_params}
Throughout this paper, we use the following model architecture for all the experiments. 

\begin{table}[h]
\centering
\caption{Model Parameters for DistIRL}
\begin{tabular}{ll}
\hline
\textbf{Parameter} & \textbf{Value} \\
\hline
\multicolumn{2}{l}{\textbf{Training Parameters}} \\
\hline
Learning Rate & $3 \times 10^{-4}$ \\
Batch Size & 512 \\
Total Iterations & 5,000 \\
Entropy Coefficient & 0.1 \\
Risk Measure & CVaR \\
Risk Parameter & 0.05 \\
Reward Regularization & 0.01 \\
\hline
\multicolumn{2}{l}{\textbf{Network Architecture}} \\
\hline
Policy Network & [256, 128] \\
Reward Distribution Family & Task-dependent (Gaussian, skew-normal, or quantile) \\
Reward Range & [-5.0, 5.0] \\
Number of Quantiles & 200 \\
Reward Hidden Features & 128 \\
\hline
\end{tabular}
\label{tab:model_params}
\end{table}
For gridworld, we specify the reward range as \( [0, 2]\). For MuJoCo tasks, \([-10, 10]\). This is achieved by applying a (scaled) tanh function.

{
\section{Additional Ablation studies}
\subsection{Ablation on choices of DRM and its parameter}
\label{appendix:drm_choices}
{
In this section, we present additional ablation studies. First, we evaluate the performance of DistIRL on the risk-averse D4RL dataset with different choices of DRM in the HalfCheetah instance. Note that for CVaR and VaR, the smaller distortion parameter $\eta$ is, the more risk-averse the policy will be. For Wang's risk measure, which has parameter $\eta$ ranging from $-1$ to $1$, the policy varies from risk-seeking to risk-averse, with $\eta=0$ having risk-neutral behavior. The choice of risk parameter affects the shape $\tilde{\xi}'$, which affects the solution quality of the policy optimization problem in Eq.~\ref{eq:policy-learning}.}

{
Table~\ref{table:ablation_drm} demonstrates the effects of different choices of risk measure and its risk parameter. Note that since the data is generated by a risk-averse policy, a risk-averse DRM produces the best result, while risk-neutral policies are substantially worse, and risk-seeking policies fail to capture the expert's behavior. 
}

\begin{table}[ht]
\small
  \caption{Performance on distributional reward settings (D4RL).}
  \label{table:ablation_drm}
  \centering
  \vspace{-0.1in}
\begin{tabular}{c|cccc|c}
    \toprule
    DRM & $\eta=0.05$ & $\eta=0.5$ & $\eta=0.9$ & $\eta=-0.5$ & $\eta=-0.9$ \\ \midrule
    CVaR & $3539.74 \pm 44.26$ & $3384.27 \pm 151.06$ & $2851.13 \pm 689.67$ & -   & - \\
    VaR      & $3539.12 \pm 76.77$ & $3423.43 \pm 113.72$  & $3081.96 \pm 522.94$  & -  & - \\
    Wang    & $2670.42 \pm 730.93$ & $2849.94 \pm 1220.71$  & $3439.46 \pm 314.48$ & $1755.25 \pm 13.42$ & $444.62 \pm 1.90$ \\
    \bottomrule
  \end{tabular}
\end{table}

\subsection{Ablation on Number of Trajectories}
\label{appendix:number_of_trajs}
\begin{table}[ht]
\small
  \caption{Performance averaged over 5 seeds for varying dataset sizes (10, 5, 3, 1 trajectories).}
  \label{table:small_datasets}
  \centering
  \vspace{-0.1in}
  \begin{tabular}{c|cccc}
    \toprule
    Environment & 10 & 5 & 3 & 1 \\ 
    \midrule
    HalfCheetah 
        & $3539.74 \pm 44.26$ 
        & $3440.67 \pm 58.48$ 
        & $3501.53 \pm 91.82$ 
        & $3238.49 \pm 339.72$ \\
    Hopper      
        & $886.44 \pm 0.79$  
        & $888.71 \pm 20.16$ 
        & $893.15 \pm 14.13$ 
        & $748.93 \pm 112.53$ \\
    Walker2d    
        & $1526.46 \pm 148.24$ 
        & $1291.44 \pm 759.45$ 
        & $1143.62 \pm 231.05$ 
        & $1151.86 \pm 180.98$ \\
    \bottomrule
  \end{tabular}
\end{table}
\vspace{-0.5\baselineskip}

{
In addition to the main comparison, we conduct an ablation study on the number of expert trajectories used to train our DistIRL algorithm. For each environment, we construct datasets with $\{10, 5, 3, 1\}$ expert trajectories, and train our method on each of these datasets independently. The evaluation protocol is kept identical to the main experiments. We report the average return over $5$ random seeds, with the standard deviation across seeds.
}

{
Table~\ref{table:small_datasets} summarizes the results. Overall, the performance degrades as the number of trajectories decreases, which is expected given the reduced coverage of the expert behavior. 
Nevertheless, our IRL algorithm remains reasonably robust in the low-data regime. With as few as $3$ to $5$ trajectories, it still achieves returns close to those obtained with $10$ trajectories on most tasks. Even in the extreme case of a single trajectory, the learned policies retain non-trivial performance, indicating that the method can extract useful structure from highly limited expert demonstrations.
}

\section{Additional Results on Matching Return Distribution}
\label{appendix:match_return_distribution}
\begin{figure}
    \centering
    \includegraphics[width=\linewidth]{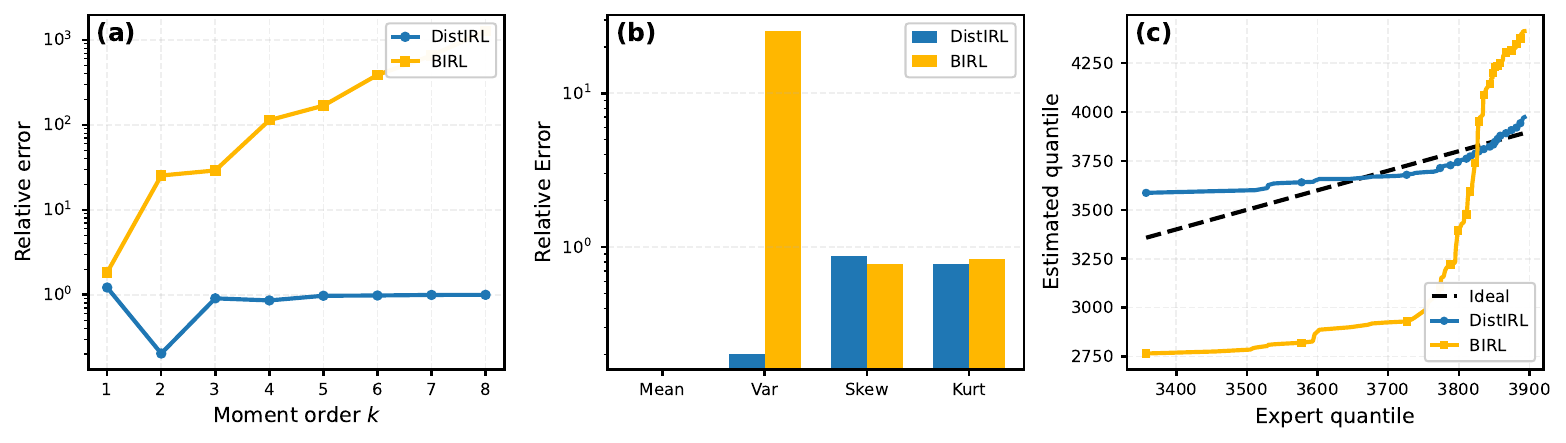}
    \caption{Distributional fidelity of learned return distributions on HalfCheetah. Left: relative errors of higher-order moments. Middle: summarized moment errors up to kurtosis. Right: estimated-versus-expert quantile alignment, where the diagonal indicates perfect matching.}
    \label{fig:return_dist_moments}
\end{figure}

Figure~\ref{fig:return_dist_moments} presents a comparison of distributional fidelity between DistIRL and BIRL using three metrics: (a) relative errors of higher-order moments, (b) summarized moment errors up to kurtosis, and (c) estimated--versus--expert quantile alignment. 
In (a), DistIRL maintains consistently low relative error across all moment orders, demonstrating its ability to capture not only the mean and variance but also the skewness and tail behavior of the expert return distribution. 
In contrast, BIRL's error grows rapidly with increasing moment order, indicating limited capacity to recover higher-order structure. 
Panel~(b) further highlights this gap, showing that DistIRL achieves uniformly low errors on the first four moments, whereas BIRL exhibits substantial discrepancies, particularly in variance and higher moments. 
Panel~(c) compares estimated and expert quantiles, where the dashed diagonal represents perfect alignment. DistIRL closely follows this ideal mapping across the entire range, while BIRL deviates significantly, especially in the upper tail. 
Overall, this figure illustrates that DistIRL reconstructs the full return distribution with higher accuracy than BIRL, which is necessary for risk-sensitive learning and downstream decision-making under uncertainty.
}

\section{Additional Results on Dopamine Level} \label{additional_results_da}
Figures in this section provide additional state-action examples from the mouse spontaneous behavior dataset. They are selected to show representative dopamine-response shapes across different syllable transitions and to complement the aggregate Wasserstein and correlation metrics in the main text. Across these examples, DistIRL is intended to recover not only the mean reward level but also the shape of the empirical dopamine fluctuation distribution.
\begin{figure}[h]
    \centering
    \includegraphics[width=0.8\textwidth]{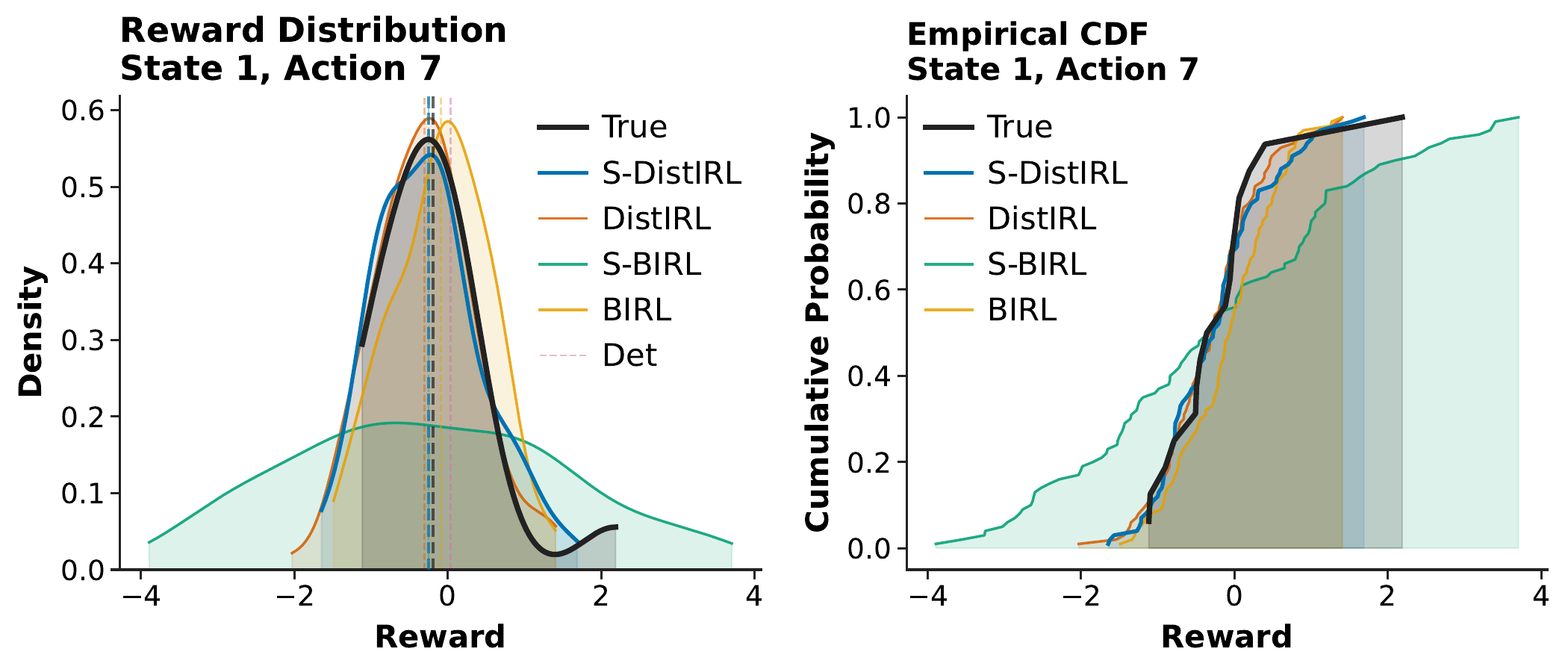}
    \caption{Reward recovery for state 1 action 7}
\end{figure}

\begin{figure}[h]
    \centering
    \includegraphics[width=0.8\textwidth]{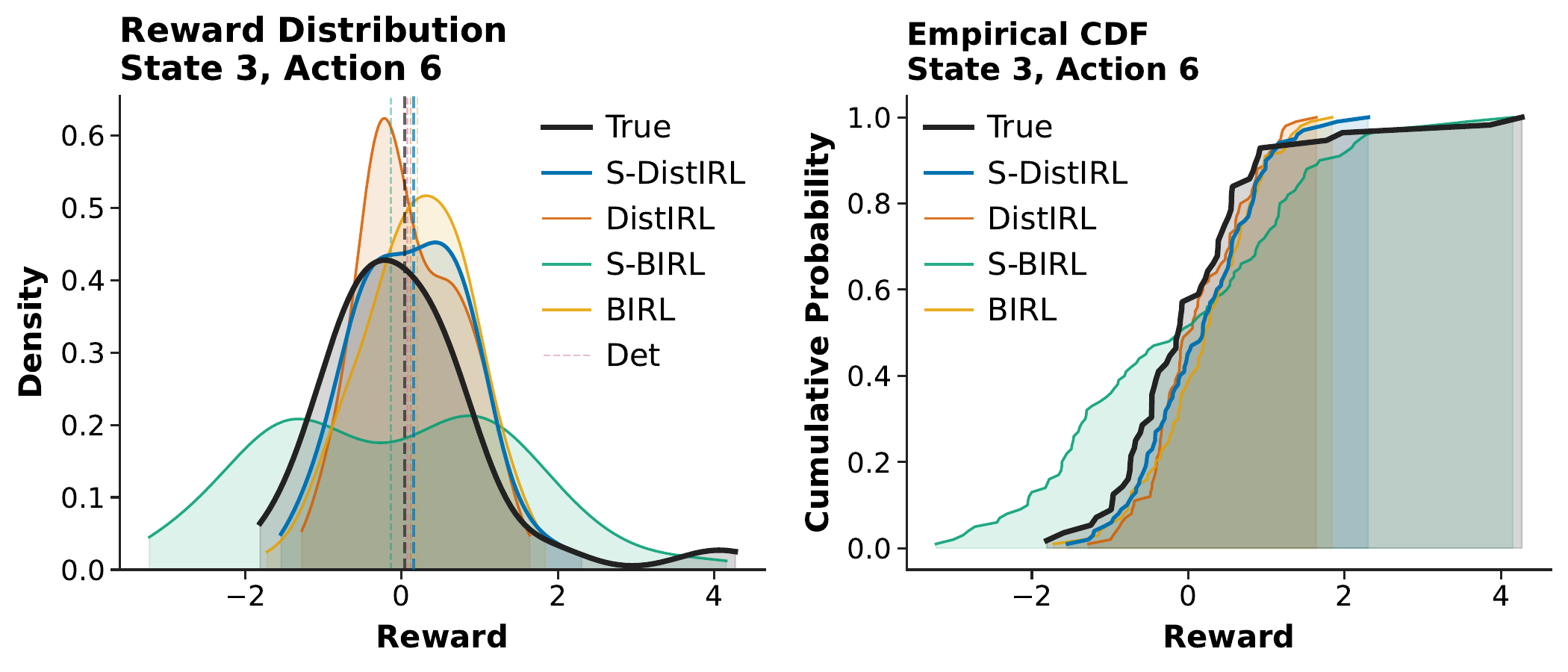}
    \caption{Reward recovery for state 3 action 6}
\end{figure}

\begin{figure}[h]
    \centering
    \includegraphics[width=0.8\textwidth]{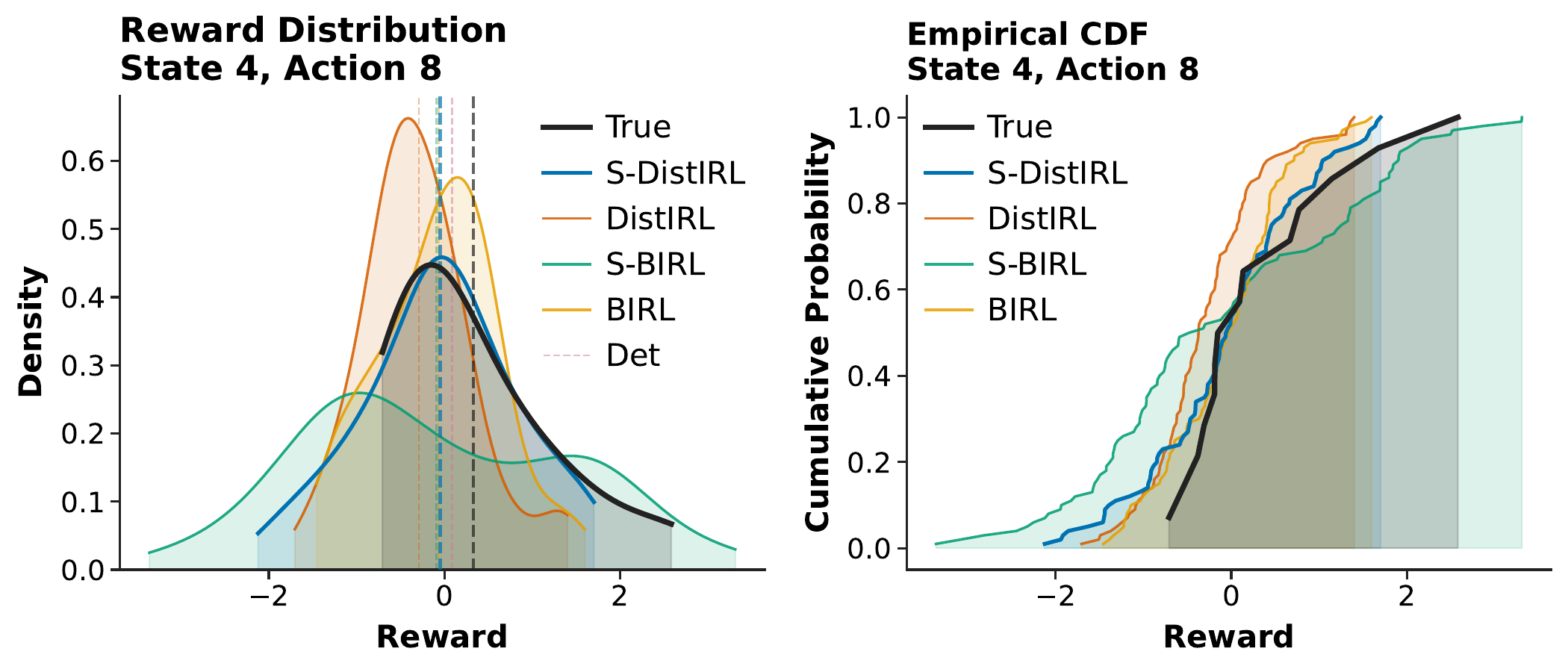}
    \caption{Reward recovery for state 4 action 8}
\end{figure}

\begin{figure}[h]
    \centering
    \includegraphics[width=0.8\textwidth]{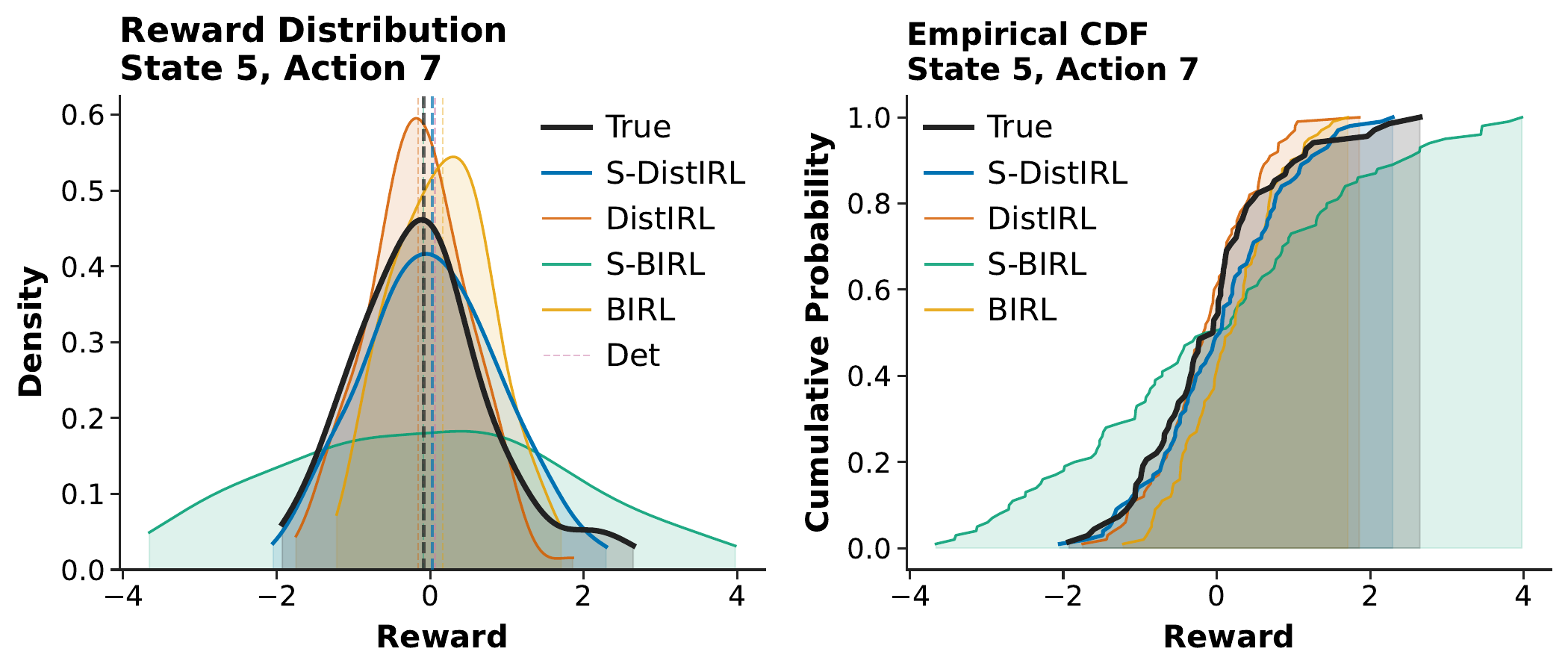}
    \caption{Reward recovery for state 5 action 7}
\end{figure}

\begin{figure}[h]
    \centering
    \includegraphics[width=0.8\textwidth]{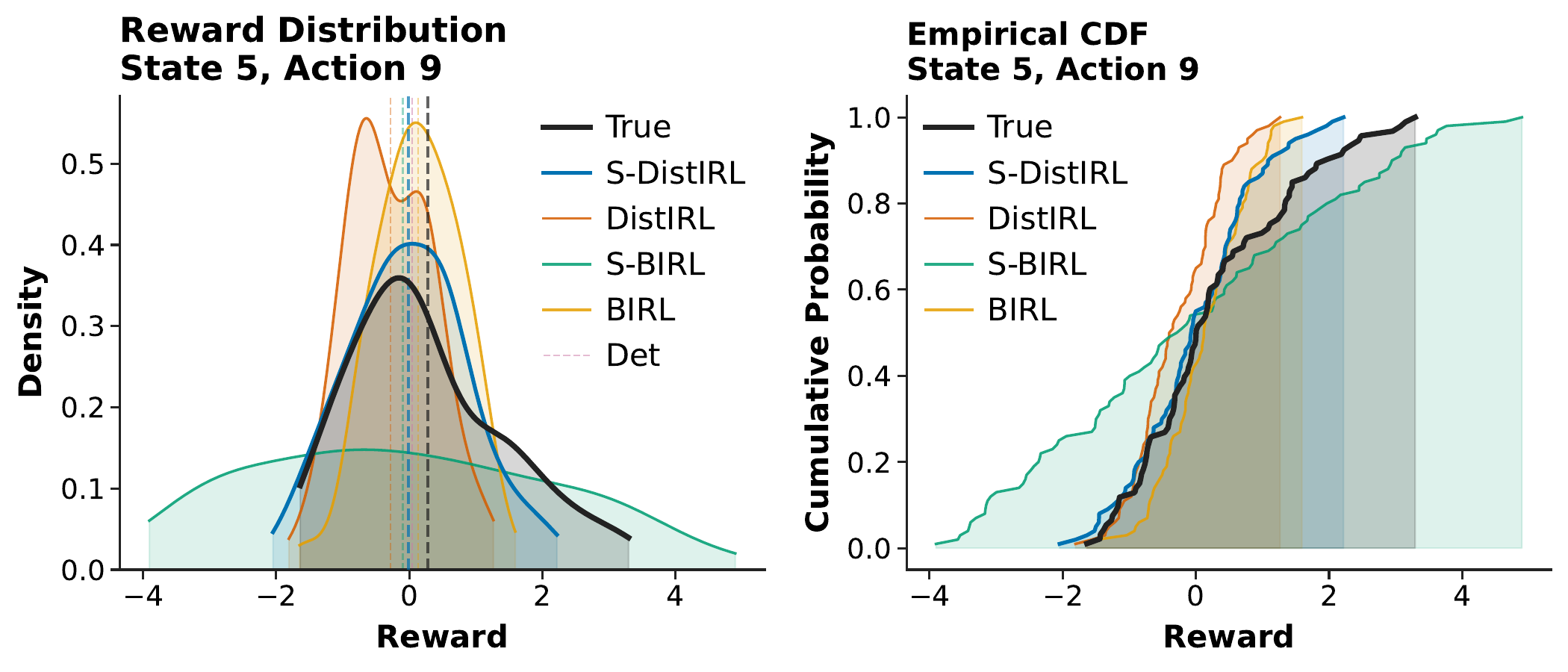}
    \caption{Reward recovery for state 5 action 9}
\end{figure}

\begin{figure}[h]
    \centering
    \includegraphics[width=0.8\textwidth]{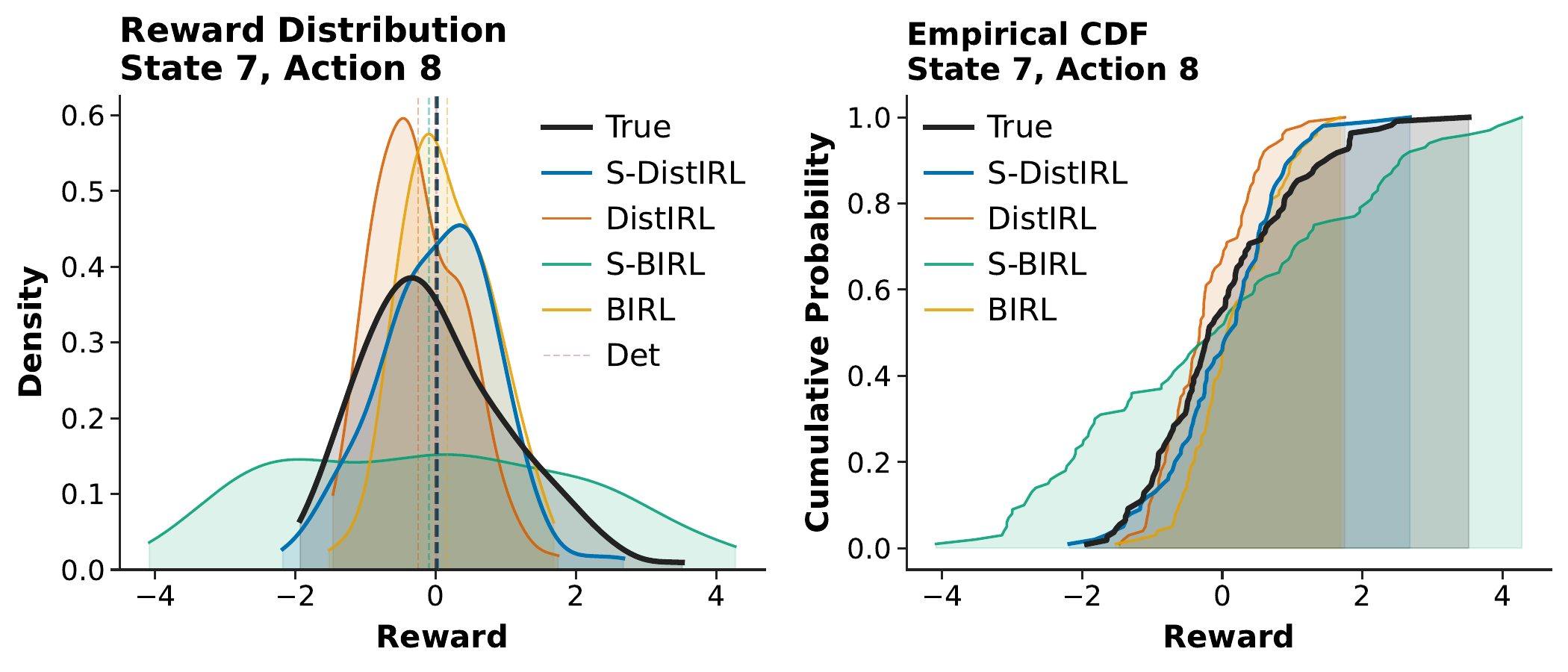}
    \caption{Reward recovery for state 7 action 8}
\end{figure}

\begin{figure}[h]
    \centering
    \includegraphics[width=0.8\textwidth]{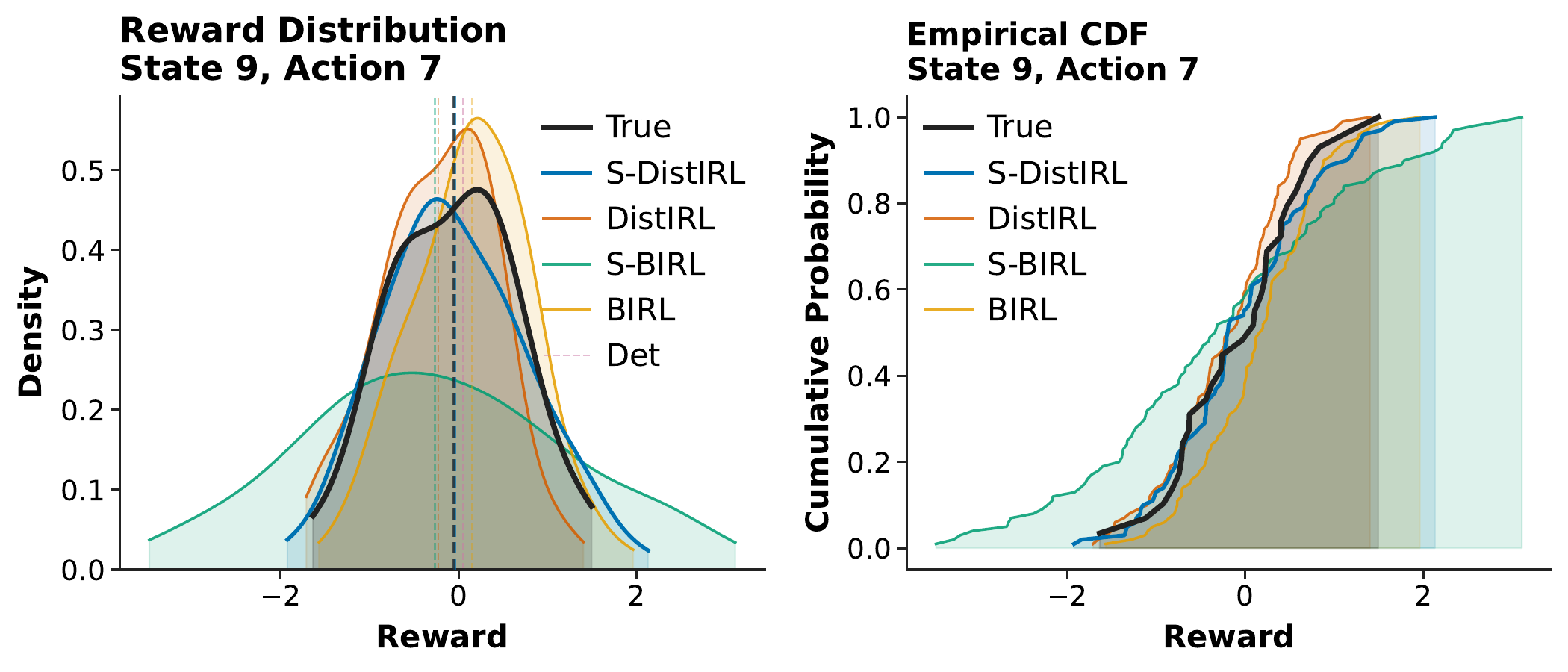}
    \caption{Reward recovery for state 9 action 7}
\end{figure}

\section{Limitations}
Our ideal policy objective (Eq.~\ref{eq:policy-learning}) enforces first-order stochastic dominance (FSD) but the indicator-based formulation is non-differentiable, making exact FSD-constrained optimization intractable in practice.
Additionally, we treat each state-action reward distribution $q_\phi(r| s,a)$ as independent. This ignores potential correlations across different $(s,a)$ pairs---such as spatial or temporal dependencies that naturally arise in many tasks. Extending DistIRL to capture joint reward distributions remains an important direction for future work.

\end{document}